\documentclass[11pt,letterpaper]{article}
\usepackage{fullpage}
\usepackage{url}

\usepackage{amssymb,amsmath,amsfonts}
\usepackage{amsmath}
\usepackage{graphicx}
\usepackage{times}
 \usepackage{ifpdf}
\usepackage{amsthm}
\usepackage{algorithm}
\usepackage{wrapfig}
\usepackage[algo2e, ruled, vlined]{algorithm2e}

\newtheorem{thm}{Theorem}[section]
\newtheorem{theorem}{Theorem}[section]

\newtheorem{lemma}[thm]{Lemma}

\newtheorem{example}[thm]{Example}

\newtheorem{corollary}[thm]{ Corollary}

\newcommand{\ignore}[1]{}

\newcommand\E{\textsc{E}}
\newcommand\V{\textsc{V}}

\newcommand{\I}{\texttt{I}}
\newcommand{\eA}{\hat{{\sf A}}}
\newcommand{\emA}{\widehat{{\sf mA}}}
\newcommand{\M}{\texttt{M}}
\newcommand{\bphi}{\boldsymbol{\phi}}
\newcommand{\boldf}{\boldsymbol{f}}
\newcommand{\btheta}{\boldsymbol{\theta}}

\newcommand{\Var}{\mathop{\sf Var}}

\newcommand{\OPT}{\textsc{OPT}}
\newcommand{\Average}{\mathop{{\texttt{{\sf Ave}}}}}
\newcommand{\Median}{\mathop{\texttt{{\sf Median}}}}

\newcommand{\tReachonearg}[2]{\texttt{R}^{#1}(#2)}

\newcommand{\RReach}{\texttt{R}} 
\newcommand{\CReach}{\texttt{Reach}} 
\newcommand{\VReach}{\texttt{VReach}} 

\usepackage[utf8]{inputenc} 
\usepackage[T1]{fontenc}    
\usepackage{hyperref}       
\usepackage{url}            
\usepackage{booktabs}       
\usepackage{amsfonts}       
\usepackage{nicefrac}       
\usepackage{microtype}      

\newcommand{\notinproc}[1]{#1}
\newcommand{\onlyinproc}[1]{}

\usepackage{xcolor} 

\SetCommentSty{mycommfont}

\title{Sample Complexity Bounds for Influence Maximization}

\author{Gal Sadeh, Tel Aviv University,\texttt{galsdh@gmail.com} \and
        Edith Cohen, Google Research and Tel Aviv University, \texttt{edith@cohenwang.com}\and
        Haim Kaplan, Google Research and Tel Aviv University, \texttt{haimk@tau.ac.il}}
 
\begin{document}
\maketitle

\begin{abstract}
Influence maximization (IM) is the problem of finding for a given $s\geq 1$ a
set $S$ of $|S|=s$ nodes in a network with maximum influence.
With stochastic diffusion models, the influence of a set $S$ of seed nodes is
defined as the expectation of its {\em reachability} over {\em
  simulations}, where each simulation specifies a deterministic reachability function.
Two well-studied special cases are the {\em Independent Cascade (IC)} and
the {\em Linear Threshold} (LT) models of Kempe, Kleinberg, and Tardos
\cite{KKT:KDD2003}.  The influence function in stochastic diffusion is unbiasedly estimated by averaging reachability values over i.i.d.\ simulations. We study the IM {\em sample complexity}: the number of 
simulations needed to determine a $(1-\epsilon)$-approximate maximizer
with confidence $1-\delta$.  Our main result is a surprising upper bound of 
$O( s \tau \epsilon^{-2} \ln \frac{n}{\delta})$ for a broad class of
models that includes IC and LT models and their mixtures, where $n$ is the
number of nodes and $\tau$ is the number of diffusion steps.  Generally $\tau \ll n$, so this significantly improves over the generic upper bound of
$O(s n \epsilon^{-2} \ln \frac{n}{\delta})$. 
Our sample complexity bounds are derived
  from novel upper bounds on the variance of the reachability that allow for small relative error for influential sets and
  additive error when influence is small.
  Moreover, we provide a data-adaptive method that can detect and utilize fewer simulations on models where it
  suffices.   Finally, we provide an efficient greedy design that computes an $(1-1/e-\epsilon)$-approximate
maximizer from simulations and applies to any submodular stochastic diffusion model 
that satisfies the variance bounds.

\end{abstract}

\section{Introduction}

Models for the spread of information among networked entities were 
studied for decades in sociology and 
economics~\cite{ThresholdModels:1978,EasleyKleinbergBook:2010,JacksonNetworks:Book2010}.
A diffusion process is initiated from a seed set of 
nodes (entities) and progresses in steps:  Initially, only the seed nodes are
activated.  In each step additional nodes may become active
based on the current set of active nodes.  The progression can be deterministic or stochastic.
The $t$-stepped influence of a seed set $S$ of nodes is then
defined as its expected reachability (total number of active nodes) in
$t$ steps.

{\em Influence maximization} (IM) is the problem of finding a set $S$
of nodes of specified cardinality $|S|=s$ and
maximum {\em influence}.  The IM problem was formulated nearly two decades ago by Richardson and
Domingos \cite{RichardsonDomingos:KDD2001,RichardsonDomingos:KDD2002}
and inspired by the application of viral marketing. 
In a seminal paper, Kempe,  Klienberg, and Tardos
\cite{KKT:KDD2003} studied stochastic diffusion models and introduced two elegant special cases, the {\em Independent Cascade
  (IC)} and {\em Generalized Threshold (GT)}  diffusion 
 models.  Their work sparked extensive followup
research and large scale implementations
\cite{MosselR:STOC2007,CWY:KDD2009,JHC:ICDM2012,NguyenTD:TON2017}.
Currently IM is applied in multiple domains with
linked entities for tasks as varied as diversity-maximization
(the most representative subset of the population) and sensor
placement that maximize coverage \cite{Leskovec:KDD2007,ChenLakshmananCastillo_book:2013,MirzasoleimanKSK:NIPS2013}.

\begin{wrapfigure}{R}{0.5\textwidth}
\vspace{-10pt}
\includegraphics[width=0.5\textwidth]{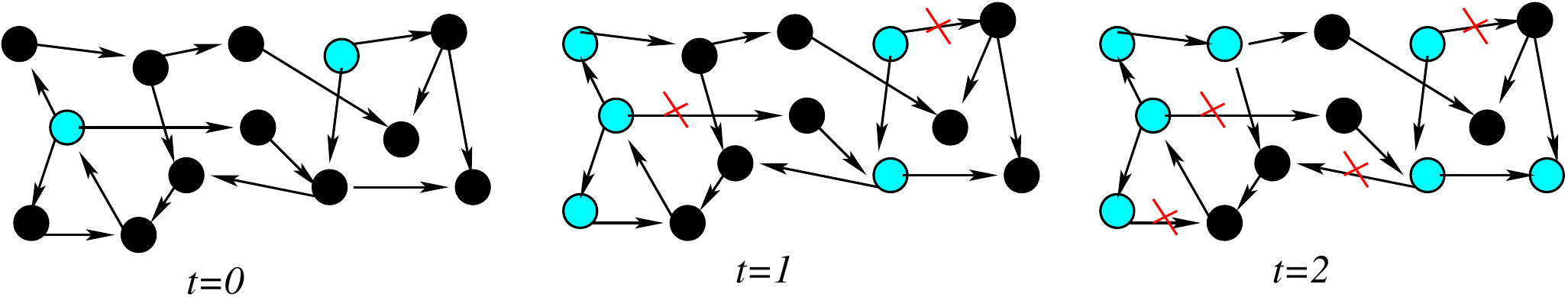}
\caption{\small A 2-step cascade from two seed nodes.
}\label{cascade:ex}
\vspace{-10pt}
\end{wrapfigure}
We consider
{\em stochastic diffusion models (SDM)}  $\mathcal{G}(V)$ over $|V|=n$ 
nodes that are specified by a distribution $\bphi \sim \mathcal{G}$ over sets
$\bphi := \{\phi_v\}_{v\in V}$
of monotone non-decreasing boolean {\em
  activation functions} 
\[\phi_v:2^{V\setminus \{v\}}\rightarrow \{0,1\} .\]
A diffusion process starts with a seed set $S \subset V$ of nodes and $\bphi \sim \mathcal{G}$.
At step $0$ we activate the seed nodes $\CReach^0(\bphi,S) := S$.
The diffusion then proceeds deterministically:
At step $t>0$ all active nodes remain active and we activate any
inactive node $v$ where
$\phi_v(\CReach^{t-1}(\bphi,S))=1$:
\[
\CReach^{t+1}(\bphi,S) := \{v\in V \mid \phi_v(\CReach^t(\bphi,S))=1\} .\]
The {\em $\tau$-steps reachability set} of a seed set $S$ is
the random variable $\CReach^\tau(\bphi,S)$ for $\bphi\sim
\mathcal{G}$ and
respectively  the $\tau$-steps {\em reachability}, $\RReach^\tau(S)$, is the
random variable that is the number of active nodes $|\CReach^\tau(\bphi,S)|$ for $\bphi\sim \mathcal{G}$.
Finally, the influence value of $S$ is defined to be the expectation
\[ \I^\tau(S) := \E[\RReach^\tau(S)] = \E_{\bphi\sim\mathcal{G}}[|\CReach^\tau(\bphi,S)|]  .\]  We
refer to the case where the diffusion is allowed to progress until there is no
growth as {\em unrestricted} diffusion  and this corresponds to
$\tau=n-1$.  The influence  $\I^\tau(S)$ is a monotone set function.
We say that an SDM is {\em submodular} when the influence function is
submodular and that it is {\em independent} if the activation functions  $\phi_v$ of different nodes are independent random variables. 
 The IM problem for seed set size $s$ and $\tau$ steps is 
to find \[\arg\max_{S: |S|\leq k}\I^\tau(S). \] 
  
 The reader might be more familiar with well-studied special cases of
 this general formulation.
{\em Live-edge} diffusion models
$\mathcal{G}(V,\mathcal{E})$ are specified by a graph $(V,\mathcal{E})$
with $|V|=n$ nodes and $|\mathcal{E}|=m$ directed edges and a
distribution $E\sim \mathcal{G}$ over subsets $E \subset \mathcal{E}$
of "live" edges.  When expressed as an SDM, the activation functions that correspond to $E$ have
$\phi_v(T)=1$  if and only if there is an edge from a node in $T$ to $v$ in the graph $(V, E)$.
Live-edge models are always submodular:  This because
$|\CReach^\tau(E,S)|$, which
\ignore{

  }
is the number of nodes reachable from $S$ in $(V,E)$ by
paths of length at most $\tau$, is a coverage function and hence 
monotone and submodular.  Therefore, so is the influence function $\I^\tau(S)$,
which is an expectation of a distribution over coverage functions.
A live-edge model is independent if we only have dependencies between incoming edges to the same node.
The Independent Cascade (IC) model is the special case of an
independent live-edge model where all
edges $e\in \mathcal{E}$ are independent Bernoulli random variables selected with probabilities $p_e$ ($e\in
\mathcal{E}$).

Another well-studied class are {\em generalized threshold} (GT) models
\cite{KKT:KDD2003, MosselR:STOC2007}.  A GT model
$\mathcal{G}(V,\boldf)$ is specified by a set $\boldf := \{f_v\}_{v\in V}$
of monotone functions $f_v:2^V\rightarrow [0,1]$.   The
randomization is specified by a set of threshold values
$\btheta\sim \mathcal{G}$ where
$\btheta :=\{\theta_v\}_{v\in V}$.  
The corresponding activation functions to $\btheta$ are
\[ \phi_v(T) := \text{Indicator}( \theta_v \leq f_v(T)) .\]
A well-studied subclass are {\em Independent GT (IGT)}  where we
require that $\boldf$ are
submodular and nodes $v\in V$ have independent threshold values
$\theta_v \sim U[0,1]$. 
Mossel and Roch \cite{MosselR:STOC2007,0612046p29:online} proved that
IGT models are submodular, which is surprising because the functions
$|\CReach^\tau(\bphi,S)|$ are generally not submodular.
Their proof was provided for unrestricted diffusion but extends to the case where we stop the process
after $\tau$ steps.
Finally, Linear threshold (LT) models \cite{ThresholdModels:1978,KKT:KDD2003}  are a
special case of IGT where we
have an underlying directed graph and each edge $(u,v)$ is associated
with a fixed weight value $b_{u v}\geq 0$ so that for all $v\in V$
$\sum_u b_{uv} \leq 1$ and the functions are defined as the sums
$f_v(A) := \sum_{u\in A \cap N(v)} b_{u v}$.  Kempe et al showed \cite{KKT:KDD2003} that each LT model is equivalent
to an independent live-edge model.

One of the challenges brought on by the IM formulation is
computational efficiency.
Kempe et al \cite{KKT:KDD2003} noted that the IM problem generalizes the classic Max Cover problem even with
$\tau=1$ and a live-edge model with a fixed set of live edges ($p_e =1$ for all $e\in \mathcal{E}$).  Therefore, IM inherits Max Cover's hardness
of approximation for ratio
better than $1-(1-1/s)^s \geq 1-1/e$ \cite{feige98} for a cover with $s$ sets.  
On the positive,
with submodular models, an approximation ratio of $1-(1-1/s)^s$ can be 
achieved by the first $s$ nodes of a greedy sequence generated by sequentially adding a node with maximum marginal value \cite{submodularGreedy:1978}.
A challenge of applying Greedy with stochastic models, however, is that even point-wise evaluation of the influence
function can be computationally intensive. Exact evaluation even for IC models is 
 \#P hard \cite{CWW:KDD2010}.   As for approximation,
Kempe et al proposed to work with {\em averaging oracles} 
\[\eA^\tau (T)  :=
\frac{1}{\ell} \sum_{i=1}^\ell |\CReach^\tau(\bphi_i,T)|\]
that
average the reachability values
obtained from a set $\{\bphi_i\}_{i=1}^\ell$ of i.i.d.\
simulations.   Recall that in the general SDM formulation, a simulation is
specified by a set $\bphi$ of node activation functions.
For live-edge models, a simulation is simply a set of
concurrently live edges $E$.  In GT models, a simulation is specified
by a set of thresholds $\btheta$.  

\begin{wrapfigure}{r}{0.20\textwidth}
  \vspace{-10pt}
\begin{example}\label{polysimu:ex}
\includegraphics[width=0.20\textwidth]{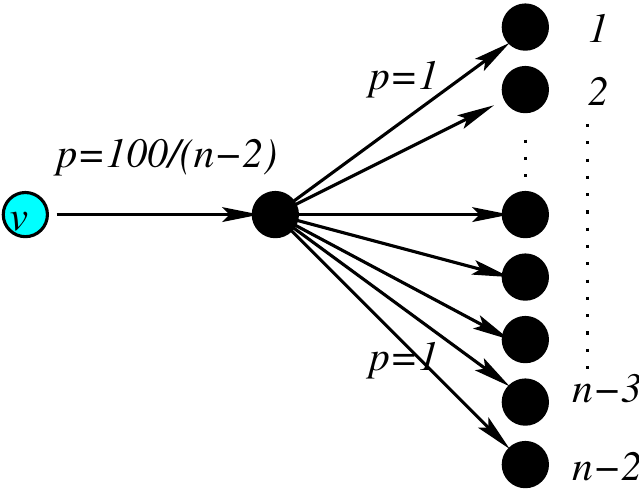}
{\small Node $v$ has influence $\I^{\tau=2}(v)=100$ but variance 
  $\approx 100n$. 
}
\end{example}
\vspace{-10pt}
\end{wrapfigure}
The averaging oracle has some appealing properties:  First, it is
robust compared to estimators tailored to models that satisfy specific assumptions (see
related work section) in that
for any diffusion model $\mathcal{G}$, also with complex and unknown dependencies
(between activation functions of different nodes or between edges in
live-edge models), for any set $S$, $\eA(S)$ is an unbiased
estimate of the exact influence value $\I^\tau(S)$ and estimates are
accurate as long as the variance of $\RReach^\tau(S)$ is
"sufficiently" small. Second, in terms of practicality, the oracle is directly available from
simulations and does not require learning or inferring the underlying diffusion
model that generated the data
\cite{SaitoNK2008,GoyalBonchi:wsdm2010,GRLK:KDD2010}.  Therefore, the
results are not sensitive to modeling assumptions and learning
accuracy~\cite{Chen:KDD2016,HeKempe:KDD2016}.  Often, estimation of model
parameters requires a large number of simulations: Even for IC models,
Example~\ref{tinymatter:ex} shows that edges with tiny probabilities that require many
simulations to estimate can be critical for IM accuracy.
 Third, in terms of computation, 
when the reachability functions $\CReach^\tau(\bphi,T)$ are monotone and submodular 
(as is the case with live-edge models), so is their average $\eA$, and hence 
the oracle optimum can be approximated by the greedy algorithm.  Prior
work addressed the efficiency of working with averaging oracles by
improving the efficiency of greedy maximization
\cite{Leskovec:KDD2007,CELFpp:WWW2011} and applied
sketches \cite{ECohen6f} to 
efficiently estimate $\eA(S)$ values
\cite{CWY:KDD2009,binaryinfluence:CIKM2014}.

The fundamental 
question we study here is the {\em sample complexity}
of IM, that is, the number of i.i.d.\
simulations needed to recover an approximate maximizer of the influence function $\I^\tau$. Formally, for parameters $(\epsilon,\delta)$, identify a seed set $T$ of size $s$ so that 
  $\Pr\left[\I^\tau(T)\geq (1-\epsilon) \OPT^\tau_s\right]\geq 1-\delta$, 
where $\OPT^\tau_s := \max_{S \mid |S|\leq s} \I^\tau(S)$ is the exact
maximum.  Note that the recovery itself is generally computationally hard
and the sample complexity only considers the information we can glean from a set of simulations.

Kempe et al provided an upper bound of
\begin{equation}\label{naive:eq}
  O\left(\epsilon^{-2} s n \log \frac{n}{\delta}\right)\  
\end{equation}
on the sample complexity of the harder {\em Uniform Relative-Error
  Estimation (UREE)} problem where for a given $(\epsilon,\delta)$  we bound the number of simulations so that with
 probability 
$1-\delta$, for all subsets $S$ such that $|S|\leq s$, $\eA(S)$ approximates $\I^\tau(S)$ within relative error of $\epsilon$.  The sample complexity of UREE upper bounds that of IM because the oracle maximizer $\arg\max_{S \mid |S|\leq s} \eA(S)$ must be an approximate maximizer.
We provide the argument for \eqref{naive:eq} here because it is basic and broadly applies to all
SDMs: The reachability values $\CReach^\tau(\bphi,S)$, and hence their expectation, $\I^{\tau}(S)$
have values in $[1,n]$. Using the multiplicative Chernoff bound (with values divided by $n$) 
we obtain that 
$O(\epsilon^{-2} n \ln \delta^{-1})$
simulations guarantee a relative error of $\epsilon$ with probability at least
$(1-\delta)$ for the estimate of any particular set 
$S$.  Interestingly, this bound is tight for point-wise estimation even for IC models:
Example~\ref{polysimu:ex} shows a family of models where $\tau=2$ and 
 $\Omega(\epsilon^{-2} n)$ simulations are required for estimating the influence value of a 
 single node. 
The UREE sample complexity bound \eqref{naive:eq} follows from applying a
union bound over all $\binom{n}{s}=O(n^s)$ subsets.
\begin{wrapfigure}{r}{0.2\textwidth}
\vspace{-18pt}  
\begin{example}\label{tinymatter:ex}
\includegraphics[width=0.2\textwidth]{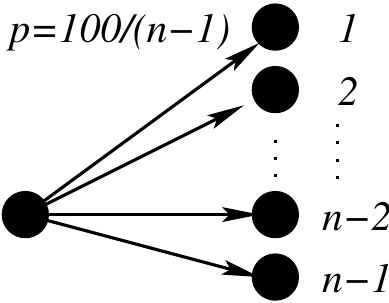}
{\scriptsize Star graph:  The center node has influence value $101$ and all other 
  nodes have influence $1$.}
\end{example}
\vspace{-16pt}
\end{wrapfigure}

The generic upper bound has prohibitive linear dependence on the
number of nodes $n$ (that Example~\ref{polysimu:ex} shows is
unavoidable for UREE even for IC models).  
A simple example shows that we can not hope for an umbrella improvement for IM:
Consider the star graph family of Example~\ref{tinymatter:ex} when
edges are dependent so that
either all edges are live or none is. Clearly $n/100$ simulations are
necessary to detect a 1-step approximate maximizer (which must be the actual
maximizer).
The remaining hope is that we can obtain stronger bounds on the IM sample
complexity for models with weaker or no dependencies such as the IC
and IGT models.
This question eluded researchers for nearly two
decades.

\subsection*{Contributions and overview}


  We study the sample complexity of influence maximization from averaging oracles computed from i.i.d.\
  simulations. 
  One of our main contributions is an upper bound of
  \begin{equation}
  O\left(\epsilon^{-2} s \tau \log \frac{n}{\delta}\right)
  \end{equation}
  on the IM sample complexity of independent strongly submodular SDMs.
  Informally, strong submodularity means that the
  influence function of any
  ``reduced'' model (model derived from original one by setting a
  subset $T\subset V$ of nodes as active) is submodular.
 The IC and IGT models are special cases of strongly submodular
 independent SDMs.

 Interestingly,  we provide similar sample complexity bounds for
 natural families of models that are not independent:  Mixtures of small number of 
 strongly submodular SDMs
and what we call
 {\em  $b$-dependence} live-edge models that allow for positive dependence of small groups
 of edges with a shared tail node.

   Our bound improves over prior work by replacing the prohibitive linear dependence in the number of nodes
  $n$ in \eqref{naive:eq}  with the typically much
  smaller value $\tau$.   While on worst-case instances unrestricted
  diffusions may require $\Omega(n)$ steps, understanding
  the sample complexity in terms of $\tau$ is important:
 First,  IM  with explicit step limits 
\cite{CLZ:AAAI2012,LiuCZ:ICDM2012,Gomez-RodriguezBS:ICML2011,DSGZ:nips2013,timedinfluence:2015},
is studied for applications where activation time matters.  Moreover, 
due to the 
``small world'' phenomenon \cite{TraversMilgram:1969}, in ``natural''
settings we can expect most
activations (even with unrestricted diffusions) to occur within a small
number of steps.  In the latter case, unrestricted influence values are
approximated well by corresponding step-limited influence with $\tau\ll n$.

Our improvement is surprising as generally a linear-in-$n$ number of
simulations is necessary for estimating influence values of some
nodes or to estimate essential model parameters (for example, the edge
probabilities in IC models), and this is the case even when $\tau$ is
very small.  This shows that the maximization problem is in an
information sense inherently
easier and can circumvent these barriers.

 We overview our results and implications --  complete proofs can be
 found in the appendix. We review related work in
 Section~\ref{related:sec} and place it in the context of our results.
 In Section~\ref{prelim:sec} we formulate quality measures for influence
 oracles and relate unrestricted and step-limited influence.
In particular, we observe that for IM it suffices that the oracle
provides good estimates (within a small relative error) of larger
influence values.  This allows us to circumvent the
lower bound  for point-wise relative-error estimates shown in Example~\ref{polysimu:ex}.

 In Section~\ref{varbounds:sec} we state our main technical result that
 upper bounds $\Var[\RReach^\tau(T)]$ by 
$\tau \I^\tau(T) \max_{v \in V 
  \setminus T} \I^{\tau-1}( v)$
 for independent strongly submodular SDMs.   This variance upper bound facilitates
 estimates with small
relative error for sets with larger influence values and
additive error for sets with small influence values.
We also provide a family of
IC models that shows that the linear dependence on $\tau$ in the
variance bound is necessary.  We derive similar variance bounds to 
mixtures of strongly submodular independent SDMs and $b$-dependence models.    All our
subsequent sample complexity bounds apply generically to any SDM
 (submodular/independent or not)  that satisfies variance bounds of this form. In Section~\ref{averaging:sec} we review {\em averaging oracles} and bound the sample complexity using variance upper bounds.
In section~\ref{moa:sec} we present our
{\em median-of-averages} oracle 
that amplifies the confidence guarantees of the averaging oracle and facilitates a tighter sample complexity bound.
\ignore{
 obtaining them
using $O(\epsilon^{-2} \tau \log
\delta^{1})$ i.i.d. simulations.  The construction uses the ``median
trick'' \cite{ams99} and organizes the simulations as
$O(\log \delta^{-1})$
pools of $O(\epsilon^{-2}\tau)$ simulations, returning the median of
the averages.
Our main result then follows by applying our amplified oracle with
$\delta $ adequate to provide a uniform error bound for all $\binom{n}{s}$ subsets of
cardinality at most $s$, which requires $O(\log({\delta^{-1}}) s \tau
\epsilon^{-2} \log n$ i.i.d. simulations.
}
In Section~\ref{adaptive:sec} we provide a data-adaptive framework that provides guarantees while avoiding the worst-case sample complexity upper bounds on models when a smaller number of simulations suffices.
\ignore{
Our simulation bounds are worst-case and in practice we can expect
that a much smaller number of simulations suffices.
Moreover, our oracles are such that ``validation'' of influence for a smaller
number of can be done
with a much smaller number of simulations.  This allows us to 
adaptively increase the number of simulations when the validation fails.
}
In Section~\ref{greedy:sec} we consider computational efficiency and present a greedy
maximization algorithm based on our median-of-averages oracles that
returns a $(1-(1-1/s)^s -\epsilon$ approximate maximizer with
probability $1-\delta$.  The
design generically applies to any SDM with a submodular
influence function that satisfies the variance bounds.
\ignore{
\\
***
Intuition for Theorem 4.1
\\
In section \ref{varbounds:sec} we show an optimal bound on the variance of the reachability of each set of nodes $T$ in IC model by:
\[ \Var[\RReach^\tau(T)] \leq \tau \I^\tau(T) \max_{v \in V 
  \setminus T} \I^{\tau-1}( v)   .\]

The main idea of the proof is inductively simulate the diffusion from $T$. For every step we define a discrete random variable $A$ that says who are the nodes we activate in the next step. We use the \textit{total variance formula} to express the reachability of $T$ as: 
\[
\Var[\tReachonearg{\tau}{T}] = {\Var}_{A}[\E[\tReachonearg{\tau}{T | A}]] + \E_{A }[\Var[\tReachonearg{\tau}{T | A}]]
 .\]

We show that thanks to the monotonicity and submodularity of the influence we can bound the first term by $\I^\tau(T) \max_{v \in V \setminus T}$. We use the second term to perform an induction step and construct a new IC model with one less step than before, the seed set in the new IC model is the value of $A$ that maximizes the variance of the reachability. Repeating this process $\tau$ times gives us the required bound.    
***
}

\section{Related work} \label{related:sec}


\begin{wrapfigure}{r}{0.4\textwidth}
\vspace{-22pt}
\begin{example}\label{dependence:ex}
\includegraphics[width=0.4\textwidth]{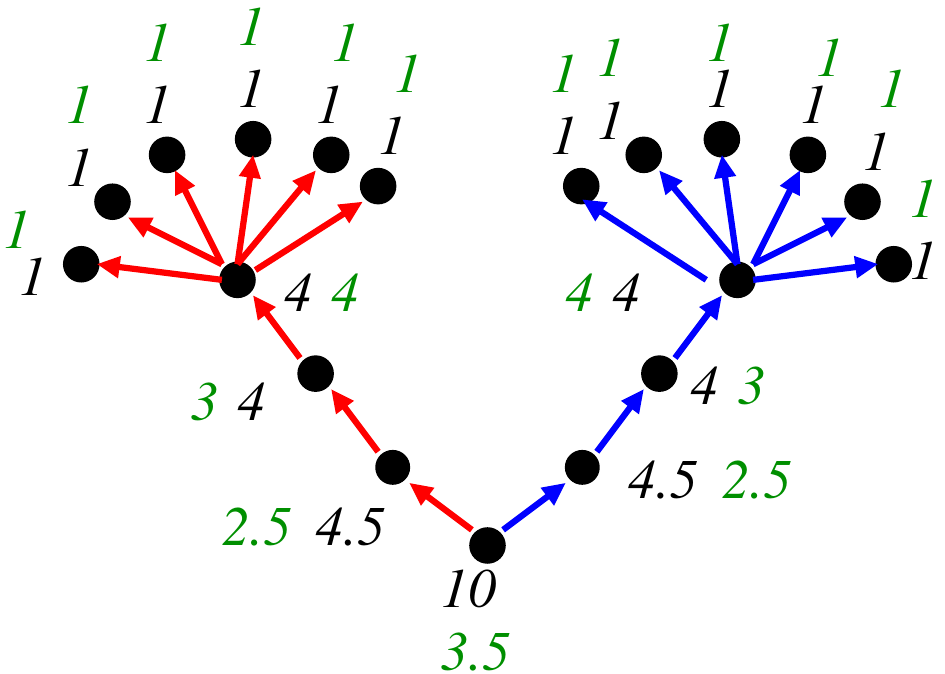}
{\scriptsize Model where with probability $\frac{1}{2}$ all red edges are active and otherwise all blue edges are active.  The influence values $\I^4(v)$ are shown in black. Simulation averages and RR samples with full simulations provide unbiased estimates of influence values $\E[\hat{\I}^4(v)]=\I^4(v)$. However, "efficient" RRS, which works with the marginal edge probabilities ($p_e=\frac{1}{2}$) or with decomposed simulations is biased, with $\E[\hat{\I}^4(v)]$ shown in green.  We can see that the bias induces large errors and also yields an erroneous maximizer.}
\end{example}
\vspace{-20pt}
\end{wrapfigure}
Our focus here is on influence estimates obtained from averages of
i.i.d.\ simulations of a model.  We note that alternative approaches 
can be more
effective for specific families of models.
In particular, for IC models,  state of the art large-scale greedy
approximate maximization algorithms 
 \cite{TXS:sigmod2014,TSX:sigmod2015,Nguyen:SIGMOD2016,HuangWBXL:VLDB2017}
are not based on simulation averages.  The estimates are also obtained by
building for each node a sample of its "influence set" but instead they use a finer building block of i.i.d.\ {\em Reverse Reachability (RR)
  searches}.
The random RR search method was proposed in \cite{ECohen6f} to 
estimate size of reachability sets in graphs 
and Borg et al \cite{BBCL:SODA2014} adapted it to 
IC models.

The method can be applied in principle for any live-edge model:
A basic RR search is conducted by selecting a node $v\in V$ uniformly at random and performing a BFS search on 
reversed edges that is pruned at length $\tau$. The search "flips" edges
as their head node is reached according to conditional  distribution on $\mathcal{G}$. The index number of the RR search is
then added to the sample set of each node that is "reached" by the search. 
Influence of a subset $S$ can then be unbiasedly estimated from the cardinality of
the union of the samples of nodes $v\in S$ and the greedy algorithm
can be applied to the sets of samples for approximate maximization. To obtain an approximate influence maximizer we need to perform RR searches 
until some node has a sample of size $O\left(\epsilon^{-2} s \log (n/\delta)\right)$.   In the worst
case,  this requires $O\left(\epsilon^{-2} s n\log (n/\delta)\right)$ RR searches. 
For general live-edge models, an independent RR search can
always be obtained
from a simulation $E\sim \mathcal{G}$ by randomly drawing
a node and performing a reverse search from it using edges $E$. The
same simulation, however, can not generally be reused to generate
multiple independent RR searches. This way of obtaining RR searches
works for general live-edge models (with arbitrary dependencies) but requires $O(\epsilon^{-2} n s \log (n/\delta))$ simulations, which does not
improve over the generic upper bound \eqref{naive:eq}.

The appeal of the RR searches method is that it can be implemented
very efficiently for independent live-edge (including IC or LT) models.
The total work performed requires only
$O(\epsilon^{-2} m s (\log (n/\delta)))$  "edge flips" that can be
easily performed using specified edge probabilities $p_e$ for IC models.
Moreover, the basic building block of RR searches are local
simulations of sets of incoming edges of specified nodes and the full
computation requires at most $O(\epsilon^{-2} s \log (n/\delta))$
local simulations for each node.   
 When we have
full simulations generated by an independent live-edge model these ``local''
simulations are independent and the required  number of "local
simulations" can be obtained by decomposing $O(\epsilon^{-2} s\log
(n/\delta))$ full simulations.   But the caveat is that this approach breaks the coherence of
simulations, as we construct each RR search from components taken from
multiple simulations. These "efficient" implementations
 (i.e. based on decomposed simulations or edge flips
according to marginal probabilities) may "catastrophically
fail" when dependencies exist: The influence estimates obtained are biased and cause
large errors even when the variance is
low. Example~\ref{dependence:ex} shows a simple mixture model (of two
degenerate IC models) where "efficient" RRS has large error due to
bias but averages of few simulations provide accurate estimates.
To summarize, with RRS, the implementation that works with full simulations is robust to dependencies but is inefficient and
the efficient implementation breaks ungracefully even with light dependencies.  
Simulations averages 
Thus we believe that 
 both basic approaches to approximate IM, simulation averages and
RRS offer distinct advantages: Simulation averages are robust in that
they remain unbiased and are accurate on any SDM, including dependent ones, 
 for which the variance is sufficiently small whereas RRS offers more efficiency with pure independence live-edge models.

 \section{Preliminaries} \label{prelim:sec}

 We consider stochastic diffusion models $\mathcal{G}(V)$ as outlined in the
 introduction.
We denote by $\CReach^\tau(\bphi,T)$ the $\tau$-steps reachability set
of $T$ when we use a specific set $\bphi$ of activation functions.   We
will use the notation $\CReach^\tau(T)$ (with the parameter $\bphi$
omitted) for the random
variable $\CReach^\tau(\bphi,T)$  obtained when we draw $\bphi \sim \mathcal{G}$ according to
the model.

\paragraph{Utility functions}
For simplicity, the discussion in the introduction took the utility of
a reachable set to be the number of  reachable nodes  $\VReach^\tau(\bphi,T) :=
|\CReach^\tau(\bphi,T)|$.  Generally, we can consider
{\em utility functions}
$H:2^V\rightarrow \Re_+$
that are nonnegative monotone non-decreasing with $H(\emptyset) =0$:
\begin{equation} \label{submodval:eq}
  \VReach^\tau(\bphi,T) := H(\CReach^\tau(\bphi,T)) . 
\end{equation}
Submodular utility is  particularly natural and studied by
Mossel and Roch \cite{MosselR:STOC2007}.
Additive utility is the special case where nodes
have nonnegative weights $w:V \rightarrow \mathcal{R}^{+}$ and 
\begin{equation} \label{additiveval:eq}
  \VReach^\tau(\bphi,T) := \sum_{v\in \CReach^\tau(\bphi,T)}w(v) .
 \end{equation}

  We consider a diffusion model $\mathcal{G}(V,H)$ together with a utility
 function $H$.  The random variable
$\RReach_{\mathcal{G}}^\tau(T)$ is the utility of the reachable set,
that is, $\VReach^\tau(\bphi,T)$  when $\bphi \sim \mathcal{G}$. The influence
function is then 
the expected utility of the reachable set
\[\I^\tau(T) := \E[\RReach^\tau(T)] = \E_{\bphi\sim
    \mathcal{G}}\VReach^\tau(\bphi,T)\ .\]
We denote the maximum influence value of a subset of cardinality
$s$ by $\OPT^\tau_s := \max_{S: |S|\leq s} \I^\tau(S)$.
 
 It follows from the definition that for any SDM $\mathcal{G}(V,H)$
 with utility $H$,  the influence $\I^{\tau}(T)$ 
is monotone non-decreasing in $\tau$ and in the set $T$ and 
 the optimum values $\OPT^\tau_s$ are non-decreasing in $\tau$ and $s$.
 Generally, influence functions $\I^\tau(T)$ of SDMs
 may not be submodular even when utility is additive.
 The influence function is submodular for 
 live-edge and for IGT models \cite{MosselR:STOC2007}
 with submodular utility.

 \paragraph{Reduced models}
 We work with the following notion of model reduction.
 Let $\mathcal{G}(V,H)$ be an independent SDM with
submodular utility.  For a set of nodes $T\subset V$, we define the 
{\em reduced model}  $\mathcal{G}'(V',H')$  of
$\mathcal{G}$ with respect to $T$:
The reduced model contains the nodes $V'=V\setminus T$.
The activation function
$\phi'_v  \sim \mathcal{G}'$ for $v\in V\setminus T$ are obtained by drawing
$\phi_v \sim \mathcal{G}$  conditioned on $\phi_v(T)=0$ and take
\[ \text{for all } S\subset V\setminus (T\cup\{v\}),\  \phi'_v(S) :=
  \phi_v(S\cup T) \]
(Note that since we have
independent SDM we can separately consider the distribution of
activation functions of each node).
The utility in $\mathcal{G}'$  is the marginal utility in
$\mathcal{G}$ with respect to $T$:
\[ \text{for all $S\subset V\setminus T$,\ } H'(S) := H(S\cup T)-H(T)
  . \]
The reduced model $\mathcal{G}'(V',H')$  is also an independent SDM with submodular utility:
Activations
functions $\bphi'\sim \mathcal{G}'$ are independent and monotone and the
utlity is monotone with $H'(\emptyset)=0$ and submodular.

 \paragraph{Strongly submodular SDM}
We say that an independent SDM  $\mathcal{G}(V,H)$ is {\em strongly submodular} if the
 utility function $H$ is submodular and the influence
 function $\I^\tau_{\mathcal{G}'}$ is submodular with any reduced model
 $\mathcal{G}'$ and step limit $\tau\geq 0$.
IC and IGT models are strongly submodular SDMs (see Theorem~\ref{ICIGTstrong:thm}).

   The variance and thus sample complexity upper bounds that we
   present in the sequel apply to any strongly submodular SDM. 
   We will also provide bounds for some dependent families of
   models. One family is a  slight generalization of IC models  that we refer to as {\em $b$-dependence}. Here
edges are partitioned into disjoint groups,  where
each group contains  at most $b$ edges emanating from the same node.  The edges in a group must be either all live or none live (are positively dependent).




\subsection{Relating step-limited and unrestricted Influence}

When unrestricted diffusion from a seed set $S$ is such that most activations occur within $\tau$ steps,  the unrestricted influence $\I(S)$ is approximated well by  $\tau$-step influence $\I^\tau(S)$. 
We can also relate unrestricted influence with small expected steps-to-activation to step-limited influence:
For a seed set $S$, node $v$, and length $d$, we denote by
$p(S,v,d)$ the probability that node $v$ is activated in a diffusion
from $S$ in step $d$.
For additive utility functions \eqref{additiveval:eq} by definition, $\I^\tau(S) = \sum_{v\in V} w(v) \sum_{d\leq \tau}  p(S,v,d)$.
The {\em expected length of an activation path} from $S$ (in
unrestricted diffusion) is:
\begin{equation}
\overline{D}(S) := \frac{\sum_{v\in V} w(v) \sum_{d\leq n}  d\cdot
  p(S,v,d)}{\I(S)}\ .
\end{equation}
The following lemma is an immediate consequence of   Markov's
inequality and shows that
$\tau$-stepped influence with $\tau=O(\overline{D}(S))$
approximates well the  unrestricted influence:
\begin{lemma} \label{unrestricted:lemma}
  For all $S$ and $\epsilon>0$,
  $\I^{\overline{D}(S) \epsilon^{-1}}(S) \geq (1-\epsilon) \I(S)$.
  \end{lemma}

\subsection{Influence Oracles} \label{sec:influence-oracle}

We say that a set function $\hat{F}$ is an {\em $\epsilon$-approximation} of another set function $F$ \textit{at a point} $T$ if
$\left| \hat{F} (T) - F(T)
      \right| \leq \epsilon  \max\{F(T),\OPT_1(F) \}$, where $\OPT_s(F) := \max_{S\mid |S|\leq s} F(S)$. That is, the
estimate $\hat{F}$  has a small relative error  for sets $T$ with $F(T) \geq \OPT_1(F)$ and a small absolute 
 error of $\epsilon \OPT_1(F)$ for sets $T$ with $F(T) \leq \OPT_1(F)$.  We say that $\hat{F}$ provides a {\em uniform} $\epsilon$-approximation for all subsets $T$ in a collection $C$ if $\hat{F}$ is an $\epsilon$-approximation for all $T\in C$.
 
An {\em influence oracle}, 
$\hat{\I}^\tau$, is a randomized data structure that is
constructed from a set of i.i.d.\ simulations of a model.
The influence oracle, $\hat{\I}^\tau$, defines a set  function (we use the same name $\hat{\I}^\tau$ for the set function) that for any input query set $T\subset
V$,  returns a value $\hat{\I}^\tau(T)$.
For $\epsilon < 1$ and $\delta<1$ we say that an oracle provides
$(\epsilon,\delta)$ {\em approximation  guarantees
with respect to $\I^{\tau}$} if for any set $T$ it is an $\epsilon$-approximation with probability at least $1-\delta$. That is
  \begin{align}
\forall  T \text{ such that} \I^{\tau}(T) \geq \OPT^\tau_1,    &
    \Pr\left[ \frac{\left| \hat{\I}^\tau (T) - \I^\tau(T)
      \right|}{\I^\tau(T) } \geq
                                                                 \epsilon\right] \leq \delta\ . \label{highpart}\\
\forall T \text{such that } \I^\tau (T)
  \leq \OPT^\tau_1, &
\Pr\left[ \left| \hat{I}^\tau (T) - \I^\tau(T) \right| \geq
                      \epsilon \OPT^\tau_1\right] \leq \delta\ .\label{lowpart}
  \end{align}                     

where $\OPT^\tau_1 :=\OPT_1({I}^\tau ) $.
Example~\ref{polysimu:ex} shows that this type of requirement is what we can
hope for with an oracle that is constructed from a small number of simulations. 

The $(\epsilon,\delta)$ requirements are {\em for each}
particular set $T$.  If we are interested in stronger guarantees that with probability
$(1-\delta)$ the approximation uniformly holds {\em for all} sets in a collection $\mathcal{C}$,
we can use an oracle that provides $(\epsilon, \delta_A=\delta/|\mathcal{C}|)$
guarantees. The $\epsilon$-approximation guarantee for all sets in $\mathcal{C}$ then follow using a  union
bound argument:
The probability that all $|\mathcal{C}|$ sets are approximated correctly is at
most $|\mathcal{C}| \delta_A \leq \delta$. 

 \section{Variance Bounds} \label{varbounds:sec}

We consider upper
bounds
 on
the variance 
$\Var\left[\RReach^\tau(T)\right]$ of the reachability of a set of 
nodes $T$ that have the following particular form 
\begin{equation} \label{factorc:eq}
\Var[\tReachonearg{\tau}{T}] \leq c \I^\tau(T) \max\{ \I^\tau(T), \max_{v \in V} \I^\tau( v)\}\end{equation} for some $c\geq 1$.  
The sample complexity bounds we present in the sequel apply to any SDM that satisfies these bounds.
In the remaining part of this section we state
our variance upper bound for strongly submodular SDMs
and extensions and a tight worst-case lower bound for IC models.

\subsection{Variance upper bound}
 
The following key theorem facilitates our main results. We show that
any strongly submodular SDM satisfy the bound \eqref{factorc:eq} with $c=\tau$. 
\notinproc{The proof is technical and provided in Appendix~\ref{varUBproof:sec}.}
\begin{theorem}[Variance Upper Bound
  Lemma]\label{var_upper_bound:thm}
Let $\mathcal{G}(V,H)$ be a strongly submodular SDM.
 Then for any step limit
$\tau \geq 0$, and a set $T \subset V$
of nodes we have
\[ \Var[\RReach^\tau(T)] \leq \tau \I^\tau(T) \max_{v \in V 
  \setminus T} \I^{\tau-1}( v)   .\]
\end{theorem}

Some natural dependent SDMs have a variance bound of the form \eqref{factorc:eq}:
\notinproc{(See Appendix~\ref{bdependence:sec} for proofs.)}
\begin{corollary} \label{extendIC:coro}
IC models with $\tau$-steps and $b$-dependence satisfy the bound \eqref{factorc:eq} with $c=2b\tau$. 
Any mixture of $\tau$-steps strongly submodular SDMs where each model has
probability at least $p$ satisfy the bound \eqref{factorc:eq}  with $c=(\tau+1)/p$.
\end{corollary}


\subsection{Variance lower bound} \label{sec:var-lower-bound}

We provide a family of IC models for which this variance upper bound is  asymptotically tight.  This shows that the 
dependence of the variance bound on $\tau$ is necessary.

\begin{theorem} [Variance Lower Bound]
For any $\tau>0$ there is an IC model
$\mathcal{G}^\tau=(V,\mathcal{E})$ with a node $v \in V$ of maximum
influence such that
$\Var[\RReach^\tau(v)] \geq \frac{1}{12} \tau \I^\tau( v)^2$
\end{theorem}  

Our family of models $\mathcal{G}^\tau=(V,\mathcal{E})$ are such that
$(V,\mathcal{E})$ is a complete directed binary tree of depth 
 $\tau \geq 1$  rooted at $v\in V$ with all edges 
directed away from the root and $p_e =
1/2$ for all $e\in \mathcal{E}$. 
We show \notinproc{(details in Appendix~\ref{varLB:sec})} that:
\begin{align*}
  \I^\tau(v)  &= \tau \\
  \Var[\RReach^\tau(v)] &=  \frac{1}{12} \tau 
                                                   (\tau-1)(2\tau-1) .
\end{align*}

\section{The Averaging Oracle} \label{averaging:sec}
The {\em averaging oracle}  uses i.i.d.\ simulations
$\{\bphi_i\}_{i=1}^\ell$. For a query $T$ it
returns the average utility of the reachability set of $T$:
$\eA^\tau (T) = \Average_{i\in [\ell]} \VReach^\tau(\bphi_i,T) :=
\frac{1}{\ell} \sum_{i=1}^\ell \VReach^\tau(\bphi_i,T)\ .$
 We quantify the approximation guarantees of an averaging oracle in terms of a variance bound of the form
  \eqref{factorc:eq}.

\begin{lemma} \label{ave:lemma}
Consider an SDM that for some $c \geq 1$ satisfies a variance bound of the form  \eqref{factorc:eq}. Then
  for any $\epsilon, \delta<1$,  an averaging oracle constructed from
  $\ell \geq \epsilon^{-2} \delta^{-1} c $ i.i.d.\ simulations 
  provides $(\epsilon,\delta)$ guarantees. 
  
  In particular for strongly submodular SDMs, we use the variance bound in Theorem
\ref{var_upper_bound:thm} and obtain 
these approximation
 guarantees using $\ell \geq \epsilon^{-2} \delta^{-1} \tau$ i.i.d.\ simulations. 
\end{lemma}
\begin{proof}
Using  variance properties of the average of i.i.d.\ random variables, we get that for any query $T$
\[\Var[\eA^\tau (T)] = \frac{1}{\ell} \Var[\RReach^\tau(T)] \leq \frac{1}{\ell} c \I^\tau(T)
\max\{\I^\tau(T), \OPT^{\tau}_1\} \ .\]
The claims follow using Chebyshev's inequality that states that for
any random variable $X$ and $M$,
$\Pr[|X-\E[X]| \geq \epsilon M] \leq
\epsilon^{-2}\Var[X]/M^2$.  We apply it to the random variable
$\eA^\tau (T)$ that has expectation   $I^\tau(T)$ and plug
in the variance bound.
To establish \eqref{highpart} we use $M= \I^\tau(T)$
and to establish \eqref{lowpart} we use $M=\OPT^\tau_1$.
\end{proof}
 \subsection{Sketched averaging oracle} \label{sketchedave:sec}
For live-edge models with additive utility
\eqref{additiveval:eq}, the query efficiency of the averaging oracle can be improved with off-the-shelf use of
$\tau$-step combined 
reachability sketches \cite{ECohen6f, 
  binaryinfluence:CIKM2014,timedinfluence:2015,ECohenADS:TKDE2015}. 
The sketching is according to a sketch-size parameter $k$ that also determines the sketches computation time and accuracy of the estimates that sketches provide. A sketch of size $O(k)$ is computed for each node $v$ so that for any set of nodes $S$, $\sum_{i=1}^r \VReach^\tau(E_i,S)$ can be efficiently estimated from the sketches of the nodes $v\in S$.  
The computation of the sketches from an arbitrary set of simulations $\{ E_i \}$ uses at most
$\sum_i |E_i| + k\sum_{v} \max_i d_v(E_i)$ edge traversals, where $d_v(E_i)$ is the maximum in-degree of node $v$ over simulations $\{E_i\}$.  In the case of an IC model, the expected number of traversals is $(k+\ell)\sum_{e} p_e$.  Sketching with general node weights can be handled as in~\cite{ECohenADS:TKDE2015}. 
The estimates obtained from the sketches are unbiased with coefficient of variation $1/\sqrt{k-2}$ and are concentrated: Sketches of size
$k=O(\epsilon^{-2}\log(\delta^{-1}))$ provide estimates with relative error $\epsilon$ with probability $1-\delta$.  

\ignore{
A sketch of size $k$ is computed for each $v\in V$ and the average 
reachability of a set $T$ can be estimated from the sketches of $v\in 
T$. 
The NRMSE of the estimate is $1/\sqrt{k-2}$ and it is also well 
concentrated. From multiplicative Chernoff bounds, the probability of the estimate exceeding 
$(1+\epsilon)$ times value is at most $e^{-\epsilon^2 k/2}$ and the estimate being below $(1-\epsilon)$ times 
  value is at most $e^{-\epsilon^2 k/(2+\epsilon)}$
(for 
$\epsilon\leq 1$). 
We can also say that we have confidence $1-\delta$ with sketch size 
$\epsilon^{-2} \log \delta^{-1}$. 
In particular, the additional variance introduced by sketching the 
average estimate is $\eA^\tau(T)^2/(k-2)$ and a choice of 
$k=O(\epsilon^{-2})$ provides similar guarantees to the unsketched 
oracle with much more efficient computation. 
Another useful property of sketches is that 
with a slight multiplicative overhead of $\log \tau$ on preprocessing time 
and sketch size, the estimate can support 
queries for $t$-stepped influence for any $t\leq \tau$. 
}
\section{Confidence Amplification: The median-of-averages oracle} \label{moa:sec}

The statistical guarantees we provide for our averaging oracle are
derived from variance bounds.  The limitation is that the
number of simulations we need
to provide $(\epsilon,\delta)$ guarantees is linear in
$\delta^{-1}$ and therefore the number of simulations we need to provide uniform guarantees (via a union bound argument) grows linearly with the number of subsets.
\ignore{
$c$ subsets we need to construct it with confidence parameter
$\delta' = \delta/c$ and apply a union bound.    The linear dependence
on $\delta^{-1}$ implies therefore a linear dependence on $c$ with is prohibitive for large values of $c$.
}
In order to find an approximate optimizer, we would like to have a uniform $\epsilon$-approximation for all
 the ${n \choose s}$ 
subsets of size at most $s$ but doing so with an averaging oracle would require too many simulations.
We adapt to our setting a classic confidence amplification technique~\cite{ams99} to
construct an oracle where the number of simulations grows logarithmically in the
confidence parameter $\delta^{-1}$.

A {\em median-of-averages} oracle is specified by a number $r$ of {\em
  pools} with $\ell$ simulations in each pool.  The
oracle is therefore constructed from
$r \ell$ i.i.d.\ simulations $\bphi_{ij}$
for $i\in [r]$ and $j\in [\ell]$.

The simulations of each pool are used in an averaging oracle that
for the $i$th pool ($i\in [r]$) returns the estimates
$\eA^\tau_i (T)$.   The median-of-averages oracle returns
the median value of
these $r$ estimates
\begin{equation} \label{MoAests:eq}
\emA^\tau(T) := \Median_{i\in [r]} \eA^\tau_i (T) =
\Median_{i\in [r]} \Average_{j\in [\ell]} \VReach^\tau(\bphi_{ij},T)\
.
\end{equation}

We establish that when the i.i.d\ simulations are  from a model that
has variance bound \eqref{factorc:eq} for some $c\geq 1$, the
median-of-averages oracle provides $(\epsilon,\delta)$
approximation guarantees using 
  $112 \epsilon^{-2} c \ln \delta^{-1}$ i.i.d.\ simulations. 
\begin{lemma} ~\label{MEoracle:lemma}
Consider an SDM that for some $c\geq 1$
satisfies the variance bound \eqref{factorc:eq}.  Then for every $\epsilon$ and $\delta$, 
  a median-average oracle $\emA$
organized with $r = 28 \ln \delta^{-1}$ pools of $\ell = 4\epsilon^{-2} c$ simulations in each
provides  $(\epsilon,\delta)$ approximation guarantees.
\end{lemma}  
\begin{proof}
  An averaging oracle with $\ell$ simulations provides $(\epsilon,\delta_A)$ approximation guarantees for $\delta_A = 1/4$. Therefore, the probability of correct estimate for any subset is at least $3/4$.
We now consider the estimates $\eA_j$ obtained from the $r$ pools
when sorted in increasing order.   The estimates that are not correct (too low or too high) will be at the prefix and suffix of the sorted order.
The expected number of correct estimates is at least $\mu \geq
\frac{3}{4} r$.  The probability that the median estimate is not
correct is bounded by the probability that number of correct estimates
is $\leq r/2$, which is $\leq \frac{2}{3}\mu$.
From multiplicative Chernoff bounds, the probability of a sum of
Bernoulli random variables beings below
$(1-\epsilon')\mu$ is at most $e^{-\epsilon'^2 \mu /(2+\epsilon')}$.
Using $\epsilon'=1/3$ 
we have
$\epsilon'^2 \mu /(2+\epsilon') = \frac{1}{9} \frac{3}{4} \frac{3}{7}
28 \ln \delta^{-1} = \ln \delta^{-1}$.
\end{proof}

As a corollary, we obtain a sample complexity bound for influence maximization from variance bounds:
\begin{theorem}  \label{simupper:thm}
Consider an SDM that satisfies the variance bound~\eqref{factorc:eq} for some $c\geq 1$. Then for any $\epsilon<1$ and $\delta<1$, using
$112 \epsilon^{-2} c s \ln \frac{n}{\delta}$ i.i.d.\ simulations we can return $T$ such that
\[\Pr\left[ \I^\tau(T) \geq (1-2\epsilon) \OPT^\tau_s \right] \geq 1-\delta\ .\]
\end{theorem}
\begin{proof}
  We construct a median-of-averages oracle with
  $\ell = 4\epsilon^{-2} c$
and
$r = 28 \ln \delta_{MA}^{-1}$ where
$\delta_{MA} = \delta /{n \choose s}$.  From Lemma~\ref{MEoracle:lemma}
using a union bound over the ${n \choose s}$ sets we obtain that with probability $1-\delta$ the oracle provides a uniform $\epsilon$-approximation for all subsets of size at most $s$. Let $S$ be a set with maximum influence $\I(S) = \OPT^\tau_s$ and
let $T$ be the oracle optimum \[T := \arg\max_{S \mid |S|\leq s}\emA(S) .\]
We have
\[
\I(T) \geq (1 - \epsilon)\emA(T) \geq (1 - \epsilon)\emA(S) \geq (1-\epsilon)^2 \I(S) \geq (1-2\epsilon)\OPT^\tau_s\ .
\]

We comment that the $(1 - 2\epsilon)$ ratio is not tight and we can obtain a bound closer to $(1-\epsilon)$.  This because the particular set $S$ to be approximated more tightly by the oracle (that uses enough simulations to support a union bound).

\end{proof}

\section{Optimization with Adaptive sample size} \label{adaptive:sec}

The bound on the number of simulations we derived in
Theorem~\ref{simupper:thm} (through a median-of-averages oracle) and
also the naive bound~\eqref{naive:eq} (for the averaging oracle) are
worst-case.  This is obtained by using enough simulations to have the
oracle provide a uniform $\epsilon$-approximation with probability at least $1 - \delta$ on any problem instance.
To obtain the uniform approximation we applied 
a union bound over ${n \choose s}$ subsets that 
resulted in an  increase in the number of required
simulations
by an $s \log n$ factor
over the base $(\epsilon,\delta)$ approximation guarantees.

On real data sets a much smaller number of
simulations than this worst-case often suffices.
We 
are interested in algorithms that adapt to such data and return a seed set of approximate maximum influence using a respectively smaller number of simulations and while providing
statistical guarantees on the quality of the end result.
To do so, we apply an
adaptive optimization framework \cite{multiobjective:2015} (some example applications are~\cite{binaryinfluence:CIKM2014,Nguyen:SIGMOD2016,topk:conext06,CCKcluster18}).
This framework consists of a ``wrapper'' that take as inputs oracle constructions from simulations and a base algorithm that performs an optimization over an oracle.  
The wrapper invokes the algorithm on oracles
constructed  using an increasing number of simulations until a
validation condition on the quality of the result is met.  
The details are provided in\onlyinproc{ the supplementary material}\notinproc{ Appendix~\ref{adaptivemore:sec}}.
We denote by $r(\epsilon,\delta)$ the number of simulations that provides $(\epsilon, \delta)$ guarantees and we obtain the following results:
\begin{theorem} \label{optAadaptive:thm}
Suppose that on our data the averaging (respectively, median-of-averages) oracle
$\hat{I}$ has the
property that with $r$ simulations, with probability at least $1-\delta$,  the oracle optimum
$T := \arg\max_{S \mid |S|\leq s}\hat{I}(S)$
satisfies
\begin{eqnarray*}
\I^\tau(T) &\geq& (1-\epsilon)\OPT_s^\tau .\\
\end{eqnarray*}
Then with probability at least
$1-5\delta$, when using
$2\max\{r,r(\epsilon,\delta)\} + O\left(\epsilon^{-2}c \left(\ln{\frac{1}{\delta}} +   \ln \left(\ln\ln \frac{n}{\delta}+ \ln s\right)\right)\right)$ simulations with the
median-of-averages oracle and
$2\max\{r,r(\epsilon,\delta)\} + O\left(\epsilon^{-2}c\left(\ln{\frac{1}{\delta}} + \ln \left( \ln\ln \frac{n}{\delta}+ \ln n \right) \right)\right)$ simulations with the averaging
oracle, the wrapper outputs a
set $T$ such that $\I^\tau(T) \geq (1-5\epsilon)\OPT_s^\tau$.
\end{theorem}

The wrapper can also be used with a base algorithm that is an
approximation algorithm.  For live-edge models, our averaging oracle is monotone and
submodular and hence we can apply greedy to
efficiently compute a set with approximation ratio at least
$1-1/e$ (with respect to the oracle).  If we use greedy as our
base algorithm we obtain the following:
\begin{theorem} \label{greedyadaptive:thm}
If the averaging oracle $\eA$ is submodular and has the
property that with $\geq r$ simulations, with probability at least $1-\delta$,  it provides a uniform  $\epsilon$-approximation for all subsets of size at most $s$,  then with
$2\max\{r,r(\epsilon,\delta)\} + O\left(\epsilon^{-2}c\left(\ln{\frac{1}{\delta}} + \ln \left( \ln\ln \frac{n}{\delta}+ \ln n \right) \right)\right)$ simulations we can find in
polynomial time a
$(1-(1-1/s)^s)(1 - 5\epsilon)$ approximate solution with confidence $1-5\delta$.
\end{theorem}

\section{Approximate Greedy Maximization} \label{greedy:sec}

In this section we consider the computational efficiency of maximization over our oracle $\hat{\I}$ that  approximates a monotone submodular influence function $\I^\tau$.
The maximization problem is computationally hard: The brute force method evaluates $\hat{\I}(S)$ on all $\binom{n}{s}$ subsets $S$ of size $s$ in order to find the oracle maximizer.
An efficient algorithm for approximate maximization of a monotone submodular function $\hat{F}$ is greedy that sequentially builds a seed set $S$ by  adding a node $u$ with maximum marginal
contribution $\arg\max_{u\in V} (\hat{F}(S\cup\{u\})-\hat{F}(S))$ at each step.  To implement greedy we only need to evaluate at each step the function on a linear number of subsets $\hat{F}(S\cup\{u\})$ for $u\in V$ and thus overall we do $sn$ evaluations of $\hat{F}$ on subsets.
With a monotone and submodular $\hat{F}$, for any $s\geq 1$ the subset $T$ that consists of the first $s$ nodes in a greedy sequence satisfies \cite{submodularGreedy:1978}:
\[\hat{F}(T) \geq ( 1-(1-1/s)^s)  \max_{S \mid |S|\leq s} \hat{F}(S)\ge (1-1/e) \OPT_s(\hat{F})\ .
\]
    If our functions $\hat{F}$ provides a uniform $\epsilon$-approximation  of another function $F$ for all subsets of size at most $s$, then $F(T) \geq (1-(1-1/s)^s)(1-2\epsilon)\OPT_s(F)$ (See the proof of \ref{simupper:thm}).
    
  The averaging oracle is monotone and submodular~\cite{KKT:KDD2003}
  when reachability functions are as in live-edge models. 
Unfortunately our median-of-averages oracle which facilitates tighter
  bounds on the number of simulations is monotone but may not be submodular
even for models where the averaging oracle is submodular. Generally when this is the case, greedy may fail (as highlighted in recent work by 
Balkanski et al~\cite{BalkanskiRS:STOC2017}).

  Fortunately, greedy is effective on a function $\hat{F}$ that is monotone but not necessarily submodular as long as $\hat{F}$ "closely approximates" a monotone submodular $F$ in that
  marginal contributions of the form \[F(u \mid S) := F(S\cup\{u\})-F(S)\] are approximated well by
  $\hat{F}(u \mid S)$ \cite{binaryinfluence:CIKM2014}.  We apply this to establish the following lemma:
  \begin{lemma} \label{almostsubmodular:thm}
  The greedy algorithm applied to a function $\hat{F}$ that is monotone and provides a uniform $\epsilon_A$-approximation of a monotone submodular function $F$ 
  where
  $\epsilon_A = \frac{\epsilon(1-\epsilon)}{14s}$ 
  returns a set $T$ such that
  $F(T)\geq (1-(1-1/s)^s)(1- \epsilon)\OPT_s(F)$.
  \end{lemma}
  Our proof of Lemma~\ref{almostsubmodular:thm}
  generally applies to an approximate oracle $\hat{F}$ of any monotone submodular function $F$\notinproc{ and is presented in Appendix~\ref{greedyproof:sec}}.
  For approximate IM we obtain the following as a corollary:
  \begin{theorem}
    Consider a submodular SDM ${\mathcal G}(V,H)$  that for some $c\geq 1$ satisfies
    the variance bound \eqref{factorc:eq}.
  Consider a median-of-averages oracle constructed with
  $O(\epsilon^{-2} s^3 c \ln \frac{n}{\delta})$ simulations of
  $\mathcal{G}$ arranged as $r=O(s\ln \frac{n}{\delta})$ pools with $\ell= O(\epsilon^{-2}s^2 c)$ simulations each.  Then with probability $1-\delta$, the set $T$ that contains the first $s$ nodes returned by greedy on the oracle satisfies $\I^\tau(T) \geq (1-(1-1/s)^s)(1-\epsilon)\OPT^\tau_s$.
   \end{theorem}
 \begin{proof}
 From
 Lemma~\ref{MEoracle:lemma}, with appropriate constants,  this configuration provides us with $(\epsilon/(14 s),\delta)$ approximation guarantees.  From Lemma~\ref{almostsubmodular:thm} greedy provides the stated approximation ratio.
\end{proof}

Greedy on the median-of-averages oracle can be implemented generically
for any SDM $\mathcal{G}$ 
by explicitly maintaining the reachability sets $\CReach(\bphi_{ij},\{v\}\cup
S)$ for all nodes $v\in V$ in each simulation $\bphi_{ij}$ as the greedy
selects nodes into the seed set $S$.
For each step, we compute the oracle value (see \eqref{MoAests:eq}) and
select $v$ for which the value for $\{v\}\cup S$ is maximized:
\[ \arg\max_{v\in V\setminus S} \emA^\tau(\{v\}\cup S) .\]
We obtain approximation guarantees, however, only when the conditions
of monotone submodular influence function and variance bounds are satisfied.
For
specific families of models, we can consider tailored efficient
implementations that
incrementally maintain reachability sets and values.

For live-edge models with additive
utility~\eqref{additiveval:eq}  we consider an implementation of greedy 
on a median-of-averages oracle.  This can be done by explicit
maintenance of reachability sets or by using
sketches~\cite{ECohen6f,binaryinfluence:CIKM2014,timedinfluence:2015,ECohenADS:TKDE2015}
(see  Section~\ref{sketchedave:sec}).  We obtain the following bounds
(proof is deferred to Appendix Section~\ref{greedymoa:sec})
\begin{theorem} \label{greedyalg:thm}
Let $\mathcal{G}$ be a live-edge model  with an additive
utility function \eqref{additiveval:eq} that satifies the
variance bound \eqref{factorc:eq}. Then
greedy on median-of-averages oracle can be implemented with explicit reachability sets in time
   \begin{equation}
   O(\epsilon^{-2} s^3 c \ln
   \left(\frac{n}{\delta}\right) \overline{m}n)\ ,
   \end{equation}
   where $\overline{m}$ is the average number of edges per simulation (For an IC model, $c=\tau$ and $\E[\overline{m}]=\sum_{e \in \mathcal{E}} p_e$).
   When using sketches, the time bound is 
   \begin{equation}
   O(\epsilon^{-2}s^3\ln\frac{n}{\delta}(c \overline{m} +s(m^*+ns)\ln n)),
   \end{equation}
   where $m^* = \sum_v \max_{ij} d_v(E_{ij})$. For an IC model, $c=\tau$ and $m^*=\sum_e p_e$ in expectation.
   \end{theorem}

\section*{Conclusion}

We explore the "sample complexity" of IM on stochastic diffusion
models and show that an approximate
maximizer (within a small relative error) can be recovered from a
small number of simulations as long as the variance is appropriately
bounded. We establish the variance bound for the large class of
strongly submodular stochastic diffusion models.  This includes IC models (where edges
are drawn independently) and IGT models (where node thresholds are
drawn independently) and natural extensions that allow for some dependencies. 
Our sample complexity bound significantly improves over the previous bounds by replacing the linear dependence in the number of nodes by a logarithmic dependence on the number of nodes and linear dependence on the length of the activation paths (which are usually very short).  
An interesting question for future work is to address the gap between
the sample complexity and the larger number of simulations currently needed for
greedy maximization.


\subsection*{Acknowledgements}
This research is partially supported by the Israel Science Foundation (Grant No. 1841/14).

\small
\bibliographystyle{plain}
\bibliography{cycle}

\begin{thebibliography}{10}

\bibitem{ams99}
N.~Alon, Y.~Matias, and M.~Szegedy.
\newblock The space complexity of approximating the frequency moments.
\newblock {\em J. Comput. System Sci.}, 58:137--147, 1999.

\bibitem{BalkanskiRS:STOC2017}
E.~Balkanski, A.~Rubinstein, and Y.~Singer.
\newblock The limitations of optimization from samples.
\newblock In {\em Proceedings of the 49th Annual {ACM} {SIGACT} Symposium on
  Theory of Computing, {STOC} 2017, Montreal, QC, Canada, June 19-23, 2017},
  2017.

\bibitem{BBCL:SODA2014}
C.~Borg, M.~Brautbar, J.~Chayes, and B.~Lucier.
\newblock Maximizing social influence in nearly optimal time.
\newblock In {\em SODA}, 2014.

\bibitem{ChenLakshmananCastillo_book:2013}
W.~Chen, L.~V.~S. Lakshmanan, and C.~Castillo.
\newblock {\em Information and Influence Propagation in Social Networks}.
\newblock Morgan \& Claypool, 2013.

\bibitem{CLZ:AAAI2012}
W.~Chen, W.~Lu, and Y.~Zhang.
\newblock Time-critical influence maximization in social networks with
  time-delayed diffusion process.
\newblock In {\em AAAI}, 2012.

\bibitem{CWW:KDD2010}
W.~Chen, C.~Wang, and Y.~Wang.
\newblock Scalable influence maximization for prevalent viral marketing in
  large-scale social networks.
\newblock In {\em KDD}. ACM, 2010.

\bibitem{CWY:KDD2009}
W.~Chen, Y.~Wang, and S.~Yang.
\newblock Efficient influence maximization in social networks.
\newblock In {\em KDD}. ACM, 2009.

\bibitem{Chen:KDD2016}
Wei Chen, Tian Lin, Zihan Tan, Mingfei Zhao, and Xuren Zhou.
\newblock Robust influence maximization.
\newblock In {\em Proceedings of the 22Nd ACM SIGKDD International Conference
  on Knowledge Discovery and Data Mining}, KDD '16. ACM, 2016.

\bibitem{ECohen6f}
E.~Cohen.
\newblock Size-estimation framework with applications to transitive closure and
  reachability.
\newblock {\em J. Comput. System Sci.}, 55:441--453, 1997.

\bibitem{ECohenADS:TKDE2015}
E.~Cohen.
\newblock All-distances sketches, revisited: {HIP} estimators for massive
  graphs analysis.
\newblock {\em TKDE}, 2015.

\bibitem{multiobjective:2015}
E.~Cohen.
\newblock Multi-objective weighted sampling.
\newblock In {\em HotWeb}. IEEE, 2015.
\newblock full version: {\tt http://arxiv.org/abs/1509.07445}.

\bibitem{CCKcluster18}
E.~Cohen, S.~Chechik, and H.~Kaplan.
\newblock Clustering small samples with quality guarantees: Adaptivity with
  one2all pps.
\newblock In {\em AAAI}, 2018.

\bibitem{binaryinfluence:CIKM2014}
E.~Cohen, D.~Delling, T.~Pajor, and R.~F. Werneck.
\newblock Sketch-based influence maximization and computation: Scaling up with
  guarantees.
\newblock In {\em CIKM}. ACM, 2014.

\bibitem{timedinfluence:2015}
E.~Cohen, D.~Delling, T.~Pajor, and R.~F. Werneck.
\newblock Distance-based influence in networks: Computation and maximization.
\newblock Technical Report cs.SI/1410.6976, arXiv, 2015.

\bibitem{topk:conext06}
E.~Cohen, N.~Grossuag, and H.~Kaplan.
\newblock {Processing Top-k Queries from Samples}.
\newblock In {\em Proceedings of the 2006 ACM conference on Emerging network
  experiment and technology (CoNext)}. ACM, 2006.

\bibitem{RichardsonDomingos:KDD2001}
P.~Domingos and M.~Richardson.
\newblock Mining the network value of customers.
\newblock In {\em KDD}. ACM, 2001.

\bibitem{DSGZ:nips2013}
N.~Du, L.~Song, M.~Gomez-Rodriguez, and H.~Zha.
\newblock Scalable influence estimation in continuous-time diffusion networks.
\newblock In {\em NIPS}. 2013.

\bibitem{EasleyKleinbergBook:2010}
D.~Easley and J.~Kleinberg.
\newblock {\em Networks, Crowds, and Markets: Reasoning About a Highly
  Connected World}.
\newblock Cambridge University Press, New York, NY, USA, 2010.

\bibitem{feige98}
U.~Feige.
\newblock A threshold of $\ln n$ for approximating set cover.
\newblock {\em J. Assoc. Comput. Mach.}, 45:634--652, 1998.

\bibitem{Gomez-RodriguezBS:ICML2011}
M.~Gomez-Rodriguez, D.~Balduzzi, and B.~Sch{\"o}lkopf.
\newblock Uncovering the temporal dynamics of diffusion networks.
\newblock In {\em ICML}, 2011.

\bibitem{GRLK:KDD2010}
M.~Gomez-Rodriguez, J.~Leskovec, and A.~Krause.
\newblock Inferring networks of diffusion and influence.
\newblock In {\em KDD}, 2010.

\bibitem{GoyalBonchi:wsdm2010}
A.~Goyal, F.~Bonchi, and L.~V.~S. Lakshmanan.
\newblock Learning influence probabilities in social networks.
\newblock In {\em WSDM}, 2010.

\bibitem{CELFpp:WWW2011}
A.~Goyal, W.~Lu, and L.V.S. Lakshmanan.
\newblock Celf++: Optimizing the greedy algorithm for influence maximization in
  social networks.
\newblock In {\em WWW}. ACM, 2011.

\bibitem{ThresholdModels:1978}
M.~Granovetter.
\newblock Threshold models of collective behavior.
\newblock {\em American Journal of Sociology}, 83(6), 1978.

\bibitem{HeKempe:KDD2016}
Xinran He and David Kempe.
\newblock Robust influence maximization.
\newblock In {\em Proceedings of the 22Nd ACM SIGKDD International Conference
  on Knowledge Discovery and Data Mining}, KDD '16. ACM, 2016.

\bibitem{HuangWBXL:VLDB2017}
K.~Huang, S.~Wang, G.~S. Bevilacqua, X.~Xiao, and L.~K.~S. Lakshmanan.
\newblock Revisiting the stop-and-stare algorithms for influence maximization.
\newblock {\em {PVLDB}}, 10, 2017.

\bibitem{JacksonNetworks:Book2010}
M.~O. Jackson.
\newblock {\em Social and economic networks}.
\newblock Princeton University Press, 2010.

\bibitem{JHC:ICDM2012}
K.~Jung, W.~Heo, and W.~Chen.
\newblock Irie: Scalable and robust influence maximization in social networks.
\newblock In {\em ICDM}. ACM, 2012.

\bibitem{KKT:KDD2003}
D.~Kempe, J.~M. Kleinberg, and {\'E}.~Tardos.
\newblock Maximizing the spread of influence through a social network.
\newblock In {\em KDD}. ACM, 2003.

\bibitem{Leskovec:KDD2007}
J.~Leskovec, A.~Krause, C.~Guestrin, C.~Faloutsos, J.~VanBriesen, and Glance N.
\newblock Cost-effective outbreak detection in networks.
\newblock In {\em KDD}. ACM, 2007.

\bibitem{LiuCZ:ICDM2012}
B.~Liu, G.~Cong, X.~Dong, and Y.~Zeng.
\newblock Time constrained influence maximization in social networks.
\newblock In {\em ICDM}, 2012.

\bibitem{MirzasoleimanKSK:NIPS2013}
B.~Mirzasoleiman, A.~Karbasi, R.~Sarkar, and A.~Krause.
\newblock Distributed submodular maximization: Identifying representative
  elements in massive data.
\newblock In {\em NIPS}, 2013.

\bibitem{MosselR:STOC2007}
Elchanan Mossel and S{\'e}bastien Roch.
\newblock On the submodularity of influence in social networks.
\newblock In {\em STOC}, 2007.

\bibitem{0612046p29:online}
Elchanan Mossel and Sebastien Roch.
\newblock 0612046.pdf.
\newblock \url{https://arxiv.org/pdf/math/0612046.pdf}, July 2009.
\newblock (Accessed on 08/16/2019).

\bibitem{submodularGreedy:1978}
G.~Nemhauser, L.~Wolsey, and M.~Fisher.
\newblock An analysis of the approximations of maximizing submodular set
  functions.
\newblock {\em Mathematical Programming}, 14, 1978.

\bibitem{Nguyen:SIGMOD2016}
H.~T. Nguyen, M.~T. Thai, and T.~N. Dinh.
\newblock Stop-and-stare: Optimal sampling algorithms for viral marketing in
  billion-scale networks.
\newblock In {\em SIGMOD}. ACM, 2016.

\bibitem{NguyenTD:TON2017}
H.~T. Nguyen, M.~T. Thai, and T.~N. Dinh.
\newblock A billion-scale approximation algorithm for maximizing benefit in
  viral marketing.
\newblock {\em IEEE/ACM Trans. Netw.}, 25(4), 2017.

\bibitem{RichardsonDomingos:KDD2002}
M.~Richardson and P.~Domingos.
\newblock Mining knowledge-sharing sites for viral marketing.
\newblock In {\em KDD}. ACM, 2002.

\bibitem{SaitoNK2008}
K.~Saito, R.~Nakano, and M.~Kimura.
\newblock Prediction of information diffusion probabilities for independent
  cascade model.
\newblock In {\em Knowledge-Based Intelligent Information and Engineering
  Systems (KES)}, volume 5179 of {\em Lecture Notes in Computer Science}.
  Springer, 2008.

\bibitem{TSX:sigmod2015}
Y.~Tang, Y.~Shi, and X.~Xiao.
\newblock Influence maximization in near-linear time: A martingale approach.
\newblock In {\em SIGMOD}, 2015.

\bibitem{TXS:sigmod2014}
Y.~Tang, X.~Xiao, and Y.~Shi.
\newblock Influence maximization: Near-optimal time complexity meets practical
  efficiency.
\newblock In {\em SIGMOD}, 2014.

\bibitem{TraversMilgram:1969}
J.~Travers and S.~Milgram.
\newblock An experimental study of the small world problem.
\newblock {\em Sociometry}, 32:425--443, 12 1969.

\end{thebibliography}
\onlyinproc{\end{document}}
\newpage
\appendix

\section{Variance upper bound: Proof of Theorem~\ref{var_upper_bound:thm}}~\label{varUBproof:sec}

In this section we prove  Theorem~\ref{var_upper_bound:thm} which
upper bounds the variance in strongly submodular SDMs.
We start by bounding
the variance in a more basic setting of a submodular function over a random subset in
Section \ref{sec:submodular} (Theorem \ref{sub-additive_variable_variance_bound}).  This will be an ingredient in our main proof provided in Section \ref{sec:4.1}.

  

We will be using the following basic tools:
\begin{lemma}\label{Total variance and total expectation}
If $X, Y$ are two random variables on the same probability space and the variance of $Y$ is finite, then:

\[ \E[Y] = \E[\E[Y|X]] \;\;\;\; \text{(total expectation)}\]
\[ \Var[Y] = \E[\Var[Y|X]] + \Var[\E[Y|X]] \;\;\;\; \text{(total variance)} \] 

where $\E[Y|X]$ is a random variable that gets the expectation of $Y$ conditioned the value of $X$ and $\Var[Y|X]$ is a random variable that gets the variance of $Y$ conditioned the value of $X$. 
\end{lemma}

When $X$ is a Bernoulli random variable $X \sim Ber(p)$ then
Lemma \ref{Total variance and total expectation} gives that
\[ \Var[Y] = \E[\Var[Y|X]] + \Var[\E[Y|X]] = p\V_1 + (1-p)V_0 + p(1-p)(\E_1-\E_0)^2.\]
where  $\V_0 = \Var[Y|X=0]$, $\V_1 = \Var[Y|X=1]$, $\E_0 = \E[Y|X=0]$, and  $\E_1 = \E[Y|X=1]$.
  
  \subsection{Submodular monotone functions on random subsets} \label{sec:submodular}
  
  Let $S = \{ a_i\}_{1< i \leq t}$ be a set with $t$ elements and let $ P = \{ p_i \}_{1< i \leq t}$ be a set of $t$ probabilities such that $p_i$ is associated with the element $a_i$. Let $X$ be a random subset of $S$ that contains $a_i$ with probability $p_i$ independently for each $i=1,..., t$. That is
\[ \forall{A \in 2^S }: \Pr[X = A] = \prod_{a_i \in A} p_i \prod_{a_i \notin A} (1 - p_i) .\]
We say that $X$ is a {\em random subset of $S$ using probabilities $P$}.
  
A \textit{submodular monotone function} $f$ over $S$ is a function with the following properties:

\begin{enumerate}

\item $f: 2^S\rightarrow R^+$

\item For every $A, B \subset S$ with $A \subset B$ and for every $x \in S \setminus B$ we have that $f(A \cup x) - f(A) \geq f(B \cup x) - f(B)$ 

\item $A \subset B \Rightarrow f(A) \leq f(B)$

\end{enumerate}

For any singelton $\{a\} \in S$ we write $f(a)$ instead of $f(\{a\})$. Let $\M_f = \max_i {f(a_i) - f(\emptyset)}$. Our purpose in this subsection is to establish the following:

\begin{theorem}\label{sub-additive_variable_variance_bound}
Let $X$ be a random subset of $S$ using probabilities $P$ and let $f$ be a submodular monotone function. Then
\[ \Var[f(X)] \leq \M_f\E\left[f(X) - f(\emptyset)\right] .\]
\end{theorem}

We give the following additional definitions and lemmas before proving this theorem.

Let $ S_{-i} = S \setminus \{a_i\} $ and let $ P_{-i} = P \setminus \{p_i\} $. We define $X_{-i}$ to be a random subset of $S_{-i}$ 
using the probabilities $P_{-i}$.
Let $ f^0_i, f^1_i $ be submodular functions over $S_{-i}$ defined by $f^0_i(A) = f(A)$ and $f^1_i(A) = f(A \cup \{a_i\})$.
Let
\[ \E_i^1 = \E[f_i^1(X_{-i})] = \sum_{A \in 2^{S \setminus \{a_i\}}} \Pr[X_{-i}=A]f^1_i(A).\] 
and
\[ \E_i^0 = \E[f_i^0(X_{-i})] = \sum_{A \in 2^{S \setminus \{a_i\}}} \Pr[X_{-i}=A]f^0_i(A).\]  

By our definitions $\E[f(X_{-i})] = \E_i^0$ and from total expectation (Lemma \ref{Total variance and total expectation}), $\E[f(X)] = p_i\E_i^1 + (1 - p_i)\E_i^0$.

\begin{lemma}\label{basis_for_sub-additive_func}

let $ f $ be a submodular monotone function over $S$ and $X$ a random subset of $S$ using probabilities $P$. Then,
\[ \forall {i}: \E_i^1 - \E_i^0 \leq f(a_i) - f(\emptyset) \leq \M_f .\]
\end{lemma}

\begin{proof}
%

Since $X$ is obtained by drawing the elements in $S$ independently it follows that

\[ \E_i^1 - \E_i^0 = \sum_{A \in 2^{S_{-i}} }\Pr[X_{-i} = A]\big[ f\left(A \cup \{a_i \}\right) - f\left(A\right)\big] \underbrace{\leq}_\text{submodularity}\] 
\[ \sum_{A \in 2^{S \setminus a_i}} \Pr\left[X_{-i}=A\right]\big[f\left(a_i\right) - f(\emptyset)\big] \leq f(a_i) - f(\emptyset) \leq  \M_f .\]
\end{proof}

\begin{lemma}\label{Max-influence-decrease_leamma}
for any submodular monotone function $f$ over $S$ and for any index $i$ we have that
$ \M_{f^0_i} \leq \M_f$ and $ \M_{f^1_i} \leq \M_f $.
\end{lemma}
\begin{proof}
The first inequality follows immediately from our definition since
\[ \M_{f^0_i} = \max_{j \neq i} {f^0_i(a_j) - f^0_i(\emptyset)} = \max_{j \neq i} {f(a_j)} - f(\emptyset) \leq \M_f. \]
For the second inequality we use submodularity as follows
\[ \M_{f^1_i} = \max_{j \neq i} {f^1_i(a_j) - f^1_i(\emptyset)} = \max_{j \neq i} {f(\{a_j \cup a_i\}) - f(a_i)} \underbrace{\leq}_\text{sub-modularity} \max_j {f(a_j)} - f(\emptyset) = \M_f. \]
\end{proof}

We are now ready for the proof of Theorem \ref{sub-additive_variable_variance_bound}.
\begin{proof} (of Theorem \ref{sub-additive_variable_variance_bound})
The proof is by induction on the size of $S$.


\textbf{Base case}: Let $S = \{ a_1\}, P = \{p_1\}$  we have that
\[ \E[f(X)] = p_1f(a_1) + (1-p_1)f(\emptyset), \]
and 
\begin{align*}
\begin{split}
\Var[f(X)] & = \E[f(X)^2] - \E^2[f(X)] = p_1f^2(a_1) + (1-p_1)f^2(\emptyset) - \big[p_1f(a_1) + (1-p_1)f(\emptyset)\big]^2 \\
 		& = p_1(1 - p_1) \big[f(a_1) - f(\emptyset)\big]^2 \le \M_f p_1(1-p_1)\big[f(a_1) - f(\emptyset)\big]. \\ 
\end{split}
\end{align*}

It is  left to prove  that $\E\left[ f(X) - f(\emptyset) \right] \geq p_1(1-p_1)\big[f(a_1) - f(\emptyset)\big]$, and indeed we have that
\[ \E\left[ f(X) - f(\emptyset) \right] = p_1\big[f(a_1) - f(\emptyset)\big] \geq p_1(1-p_1)\big[f(a_1) - f(\emptyset)\big]. \]

\textbf{Inductive Step}: Assume the lemma holds for sets of size $\ell$ and any submodular function $f$ and probabilities $P$.
 For a set $S$ with $\ell + 1$ elements and a submodular function $f$ over $S$. Let $j\leq i$ be an arbitrary index.

From the total variance formula in Lemma \ref{Total variance and total expectation} we know that
\begin{equation} \label{eq_5}
\Var[f(X)] = p_j\V_j^1 + (1-p_j)\V_j^0 + p_j(1-p_j)\left[\E_j^1 - \E_j^0\right]^2 \ ,
\end{equation}

where
$ \E_j^1 = \E[f^1_j(X_{-j})]$,
$ \E_j^0 = \E[f^0_j(X_{-j})]$,
 $ \V_j^1 = \Var[f^1_j(X_{-j})]$,
 and $\V_j^0 = \Var[f^0_j(X_{-j})]$.
 
By applying the induction hypothesis to  
 $S_{-j}$ with  probabilities $P_{-j}$ and $|S_{-j}| = \ell$ and 
 $f_j^0$ and $f_j^1$ we get that
\[ \V_j^0 \leq \M_{f_j^0} \big[\E_j^0 - f_j^0(\emptyset)\big] \underbrace{\leq}_\text{Lemma \ref{Max-influence-decrease_leamma}} \M_{f} \left[\E_j^0 - f(\emptyset)\right],\]
and 
\[ \V_j^1 \leq \M_{f_j^1} \big[\E_j^1 - f_j^1(\emptyset)\big] \underbrace{\leq}_\text{Lemma \ref{Max-influence-decrease_leamma}} \M_{f} \left[\E_j^1 - f(a_j)\right].\]

Substituting these bounds  in Equation (\ref{eq_5}) we get that
\begin{align*}
\begin{split}
\Var[f(X)] 	& \leq p_j\M_{f}\big[\E_j^1 - f(a_j)\big] + (1-p_j)\M_{f}\big[\E_j^0-f(\emptyset)\big] + p_j(1-p_j)\big[\E_j^1 - \E_j^0\big]^2 \\
			& = \M_f\big[p_j\E_j^1 + (1-p_j)\E_j^0 - f(\emptyset)\big] + p_j(1-p_j)\big[\E_j^1 - \E_j^0\big]^2 - p_j\M_{f}\big[f(a_j) - f(\emptyset)\big]\\
			& \underbrace{=}_\text{total expectation} \M_f\big[\E[f(X)] - f(\emptyset)\big] + p_j(1-p_j)\big[\E_j^1 - \E_j^0\big]^2 - p_j\M_{f}\big[f(a_j) - f(\emptyset)\big] \\
			& \underbrace{\leq}_\text{Lemma \ref{basis_for_sub-additive_func}} \M_f\big[\E[f(X)]-f(\emptyset)\big] + p_j\M_f\big[f(a_j)-f(\emptyset)\big]\big[(1-p_j) - 1\big]\\
			& \leq \M_f[\E[f(X)]-f(\emptyset)].
\end{split}
\end{align*}

\end{proof}

\subsection{Properties of reduced diffusion models}

  We establish some properties of reduced independent SDMs that are
  needed for our upper bound. 

 We first show that influence values of nodes in a reduced
  model can only be lower than
  respective values in the original model:
  \begin{lemma} \label{infrel:lemma}
    Let $\mathcal{G}'(V\setminus T,H')$ be a reduction of a model $\mathcal{G}(V,H)$.
\begin{equation} \label{infrel}
\text{For all $v\in V\setminus T$ and $t\geq 0$,\ } \I^{t}_{\mathcal{G}'}(v)
\leq \I^{t}_{\mathcal{G}}(v) . \end{equation}
\end{lemma}
\begin{proof}
Note that $\mathcal{G}'$ is obtained from $\mathcal{G}$ by removing
nodes.  Therefore respective reachability sets given $\bphi$  are such
that those in
$\mathcal{G}'$ can only
be subsets of those in $\mathcal{G}$:
\[\CReach_{\mathcal{G}'}(v,\bphi') \subset
  \CReach_{\mathcal{G}}(v,\bphi)  .\]    Then from monotonicity and
submodularity of $H$ we get
\[
H'(\CReach_{\mathcal{G}'}(v,\bphi')) \leq
H'(\CReach_{\mathcal{G}}(v,\bphi)) \leq
H(\CReach_{\mathcal{G}}(v,\bphi)) .
\]
(second inequality follows from
monotonicity and submodularity of $H$ 
so that for all $A\subset V\setminus T$,  $H'(A) \leq H(A)$.)
  Therefore,
 \[ \I^{t}_{\mathcal{G}'}(v) = \E[ \VReach^t_{\mathcal{G}'}(v)] \leq
  \E[ \VReach^t_{\mathcal{G}}(v)] = \I^{t}_{\mathcal{G}}(v) .\]
\end{proof}

 A convenient property is that  reduction preserves strong monotone
 submodularity:
\begin{lemma} \label{reductionstrong}
A reduction of a strongly monotone submodular model is also 
strongly monotone submodular.
\end{lemma}
\begin{proof}
A reduced model with respsect to $T_2$ of a reduced model
of $\mathcal{G}$ with respect to $T_1$ is a reduced model of
$\mathcal{G}$ with respect to $T_1\cup T_2$.  Also note that the
reduced utility function $H'$ is also monotone and submodular.
\end{proof}

We next show that IC or IGT models with submodular utility are closed
under reduction:
\begin{theorem} \label{ICIGTstrong:thm}
IC and IGT models with submodular utility are 
strongly submodular SDMs.
\end{theorem}
\begin{proof}
We first show that IC/IGT models are independent SDMs.  
In the introduction we expressed IC and IGT models as SDMs:
A live-edge model is expressed as an
SDM using $\phi_v(T)=1$ if and only if there is an edge from a node in
$T$ to $v$.  The model is independent if for all $v$ the edges
incoming to $v$ are independent of all other edges.  In IC models all
edges are independent and hence IC models are independent SDMs.
Recall (from the Introduction) that
an IGT model is expressed as an SDM using $\phi_v(T) :=
\text{Indicator}(\theta_v \leq f_v(T))$.   In an IGT model the
thresholds $\theta_v$ are independent random variables, and hence
$\phi_v$ are independent.  Hence, an IGT model is an independent SDM.
Submodularity of influence when utlity is submodular is established
for IC models in~\cite{KKT:KDD2003} and for IGT models in~\cite{MosselR:STOC2007}.

 Reduction of any model preserves submodularity of the utility and
 in particular this holds for reduced IC/IGT models.  
  What remains to show is that a reduced IC/IGT  model is
  also an IC/IGT model (respectively).  This would conclude the proof
  of strong submodularity since any IC and IGT
 models with submodular utility has a
 submodular influence functions.

  To establish this remaining claim we consider IC/IGT models and express
the reduction in terms of the activation functions as one in
terms of the respective family of models.

We first consider IC models.
The reduced IC model $\mathcal{G}'(V\setminus T, \mathcal{E}\setminus
(V\times T \cup T\times V)$ is obtained from $\mathcal{G}(V)$ by deleting
the nodes $T$ and their incident edges and keeping $p_e$ on remaining
edges.  This is clearly an IC model.   It remains to show that this is
equivalent to the reduction of the distribution of activation
functions.  The conditioning that $\phi_v(T)=0$ is equivalent to
live-edge set $E$ with no edges from $T$ to $V$.   For such edge set for
any $S\subset V\setminus (T\cup\{v\})$ we have 
$\phi'_v(S) = \phi_v(S\cup T) = \phi_v(S)$ which corresponds to $E$
having at least one edge from $S$ to $v$.   From independence of
edges, the conditional distribution is also independent and retains
the same inclusion probabilities.

We next consider IGT models.
The reduction $\mathcal{G}'(V\setminus T,
  \{f'_v\})$ in terms of activation functions distribution is
  equivalent to 
  functions is equivalent to modifying the functions so that 
\[
f'_v(S) := f_v(S\cup T) - f_v(T) .
\]
The reduced model is clearly an IGT model. 
The conditioning that $\phi_v(T) = 0$ means that $\theta_v > f_v(T)$.  Therefore, the conditional distribution of $\theta_v$ provided it was not activated in the first step is uniform on $[f_v(T),1]$.   
The probability that $\theta_v > f_v(S\cup T)$ given this conditioning
is equal to the probability that $\theta'_v > f_v(S\cup T)-f_v(T) = f'_v(S)$.
 \end{proof}

\subsection{Upper bound on the variance in strongly
  submodular SDM}
\label{sec:4.1}

Let $\mathcal{G}(V,H)$ be a $\tau$-stepped diffusion model. We denote by $\M^\tau_{\mathcal{G}}(\bar{T})$ the maximum influence of a single node in $\mathcal{G}$ that is not included in $T$:
\[ \M^\tau_{\mathcal{G}}(\bar{T}) = \max_{v \in V \setminus T} \I_\mathcal{G}^\tau(v) \]

As before, we omit $\mathcal{G}$ if it can be understood from the context. We prove the following theorem which is a restatement of
Theorem \ref{var_upper_bound:thm}.

\begin{theorem}\label{bound variance in graph}
Let $\mathcal{G}(V,H)$  be a strongly submodular SDM.
 Then for any
$\tau \geq 0$ and a set of nodes $T\subset V$:
\[ \Var[\RReach^\tau(T)] \leq \tau \M^{\tau-1}(\bar{T})\I^\tau(T) .\] 
\end{theorem}

The remaining part of this Subsection contains the proof of the Theorem.


Let $T$ be a set of nodes, and let
\[N(T) = \{v\in  V\setminus T \mid \Pr[\phi_v(T)=1]>0\} \]
be the nodes that have nonzero probability to be activated if $T$ is active.
For the special case of IC models, $N(T) = \left\{ v \notin T \mid \exists (u, v) \in \mathcal{E}, u \in T \right\} $ is the set of outgoing neighbors of $T$.

We first consider the case where $N(T)$ is empty.  In this case,
$\CReach^\tau(T) = T$ for all $\tau\geq 0$.  Therefore, 
$\Var[\RReach^\tau(T)]=0$, $\I^\tau(T)=H(T)\geq 0$, and $\M^{\tau-1}(\bar{T})=0$ and the claim holds.

We now assume that $N(T)$ is not empty and give a proof by induction on $\tau$.

\subsubsection{Base case ($\tau=1$)}
Let \[p_v := \Pr [\phi_v(T)=1]\] be the probability that node $v$ is activated in step $1$ provided that the set of nodes $T$ was active at step $0$.  From independence of the model, the events of activating different nodes $v \in N(T)$ at step 1 are independent.
We have that the set of nodes that is active at step 1 is a random subset $S$ of $N(T)$ with probabilities $\{p_v\}$ as defined in Subsection~\ref{sec:submodular}.   Moreover, from monotonicity and submodularity of $H$, the function 
$f(S) :=  H(T\cup S)-H(T)$ is monotone and submodular with $f(\emptyset)=0$.  
We can therefore apply
Theorem~\ref{sub-additive_variable_variance_bound} to bound the
variance of $f(S)$:
\begin{equation} \label{fbound}
  \Var[f(S)]\leq M_f \E[f(S)].
  \end{equation}
We now note that
\[ \E[f(S)] = \E[H(T\cup S)] -H(T) = \I^1(T)-I^0(T)\] and
\[\Var[f(S)]=\Var[H(T\cup S)]= \Var[\RReach^1(T)] .\]
For all $v\in N(T)$ we have $f(v) = H(T\cup\{v\})-H(T) \leq H(v) = \I^0(v)$.  Therefore
\[M_f := \max_{v\in N(T)} f(v) \leq  \max_{v\in N(T)} \I^0(v) = M^0(\bar{T}) .\]
Substituting in \eqref{fbound}  we obtain the claim
\[ \Var[\RReach^1(T)] \leq  M^0(\bar{T}) \I^1(T) . \]

\ignore{
{\bf special cases *******}

For the special case of  IC models with a uniform additive utility we have 
\[p_v = 1 - \prod_{u\in T \mid e=(u,v)\in\mathcal{E}}(1-p_e) \] noting that  $v$ is activated only if one of the edges from $T$ to $v$ is live. Clearly, $\I^1(T) = |T| + \sum_{v \in N(T)} p_v$.  Note that the events of activating different nodes $v \in N(T)$ at step 1 are independent because the edge sets that support activation of different nodes are disjoint. Therefore, $\Var[\tReachonearg{1}{T}] = \sum_{v \in N(T)} p_v(1-p_v) $. 
 We get that
\[\Var[\tReachonearg{1}{T}] =  \sum_{v \in N(T)} p_v(1-p_v) \leq |T| + \sum_{v \in N(T)} p_v = \I^1(T) \leq M^0(\bar{T}) \I^1(T)  .\]
The last inequality follows since
$M^0(\bar{T}) = 1$ (using that $|N(T)|\geq 1$).


With weighted additive utility functions, the 1-step influence and variance are:
\begin{eqnarray*}
 I^1(T) &=& \sum_{v \in T} w_v + \sum_{v \in N(T)} p_v w_v \\
 \Var[\tReachonearg{1}{T}] &=& \sum_{v \in N(T)} p_v(1-p_v)w_v^2 \ .
 \end{eqnarray*}
Noting that
$ M^0(\bar{T}) \geq \max_{v \in N(T)} w_v$, we obtain that 
\[ \Var [\tReachonearg{1}{T}] \leq (\max_{v \in N(T)} w_v) \sum_{v \in N(T)} p_v(1-p_v)w_v\leq  \M^0(\bar{T})\sum_{v \in N(T)} p_v w_v \leq  \M^0(\bar{T})\I^1(T)\ .
\]

{\bf ****}
}
\subsubsection{Inductive step}

We define $\CReach^{t}_{\mathcal{G}}(T \mid A)$ to be the random variable
that is the $t$-steps reachability of $T$ in a diffusion on
$\mathcal{G}$ seeded with $T$ that is conditioned on the event that
exactly the nodes $A\subset N(T)$ (and no other nodes) are
activated in step 1.  Equivalently, we condition on $\bphi$ such that 
for $v\setminus (T\cup A)$, $\phi_v(T)=0$ and for $v\in A$, $\phi_v(T)=1$.
We respectively define
$\RReach^{t}_{\mathcal{G}}(T \mid A)$ to be the random variable
$H(\CReach^{t}_{\mathcal{G}}(T \mid A))$.
From definition, we have
\begin{equation} \label{aggA:eq}
\I^t_{\mathcal{G}}(T) = \sum_{A\subset N(T)}
\Pr[\CReach^1_{\mathcal{G}}(T) = A\cup T]
\E[\RReach^{t}_{\mathcal{G}}(T \mid A)] = \E_A \E[\RReach^{t}_{\mathcal{G}}(T \mid A)] \ .
\end{equation}  

We consider the reduced model $\mathcal{G}'$ of $\mathcal{G}$
with respect to $T$ and show that
the conditioned $t\geq 1$ steps  diffusion from $T$ in $\mathcal{G}$ is
equivalent to the unconditioned $t-1$ steps diffusion from $A$ in $\mathcal{G'}$:
\begin{lemma} \label{samedist}
 For any $A\subset N(T)$ and $t\geq 1$, the random variables $\CReach^{t-1}_{\mathcal{G}'}(A)$ and 
$\CReach^{t}_{\mathcal{G}}(T|A) \setminus \{T\}$ have identical
distribution overs subsets.   The random variables
$\RReach^{t-1}_{\mathcal{G}'}(A)$ and 
$\RReach^{t}_{\mathcal{G}}(T|A) -H(T)$ have identical distributions
over values.
\end{lemma}
\begin{proof}
 We first consider $t=1$. For a draw of conditioned activation functions we have
$\CReach^1 _{\mathcal{G}}(T|A) = T\cup A$.  By definition, we also have 
$\CReach^{0}_{\mathcal{G}'}(A) =A$ and the claim holds.

We next consider $t>1$. 
We first observe that 
in both situations, (i) the reduced model $\mathcal{G}'$  when seeded with $A$ and (ii) the conditioned diffusion in $\mathcal{G}$ seeded with $T$ such that the nodes $A$ are activated in the first step, the progression is determined only by the activation functions on the nodes $V\setminus (T\cup A)$ .

We next argue that the distribution of activation functions projected on the nodes $V\setminus (T\cup A)$ is the same in both situations. 
From independence of $\mathcal{G}$ it suffices to consider separately the activation functions of each node.  From definition of a reduced model, we draw for each $v\in V\setminus T$,   $\phi_v\sim \mathcal{G}$ conditioned on $\phi_v(T)=0$.
This is exactly what we get for the conditioned diffusion in $\mathcal{G}$.

We can thus match the supports (sets of activations functions) in both situations so that 
$\bphi$ and $\bphi'$ are matched when the 
projections on $V\setminus (T\cup A)$ is the same.  The starting points are at steps $0$ of the reduced model and step $1$ of the conditioned process is $A$, the progression of new activations is thus the same.
 Therefore, for any step $t \geq 1$,
\[ \CReach^t _{\mathcal{G}}(T \mid A,\bphi) \setminus T = \CReach^{t-1}
  _{\mathcal{G}'}(A,\bphi') \]
and the first claim follows.

For the second claim, note that 
$\RReach^{t}_{\mathcal{G}}(T|A) = H(\CReach^{t}_{\mathcal{G}}(T|A))$ and thus
\[\RReach^{t-1}_{\mathcal{G}'}(A) = H'(\CReach^{t-1}_{\mathcal{G}'}(A)) =
H(\CReach^{t-1}_{\mathcal{G}'}(A) \cup T) - H(T) =
H(\CReach^{t}_{\mathcal{G}}(T|A)) - H(T)\] where the equalities are
those of distributions.
\end{proof}  

As immediate corollaries we can relate expectations  and variance of
 as follows:
\begin{equation}\label{removing step}
   \I^{\tau-1}_{\mathcal{G}'}(A)=\E[\RReach^{\tau-1}_{\mathcal{G}'}(A)]=
   \E[H'(\CReach^{\tau-1}_{\mathcal{G}'}(A))]=
   \E\left[H(\CReach^{\tau}_{\mathcal{G}}(T | A))\right] - H(T) =
   \E[\RReach^{\tau-1}_{\mathcal{G}}(T | A)]-H(T) \ .
\end{equation}
\begin{eqnarray}
  \Var[\RReach^{\tau-1}_{\mathcal{G}'}(A)] & = &   \Var[H'(\CReach^{\tau-1}_{\mathcal{G}'}(A))]=
  \Var[H'(\CReach^{\tau}_{\mathcal{G}}(T | A) \setminus
                                                 \{T\})]  \label{varrel}
  \\ &=&
    \Var[H(\CReach^{\tau}_{\mathcal{G}}(T | A)) -
    H(T)]=\Var[H(\CReach^{\tau}_{\mathcal{G}}(T | A))]=
         \Var[\RReach^{\tau}_{\mathcal{G}}(T|A)]  \nonumber \ .
\end{eqnarray}

\paragraph{Total Variance:}
We define the random variable $A$ to be the subset of $N(T)$ which is activated after the first step. Note that $A$ is a random subset of $N(T)$ using probabilities $p_v$ for $v\in N(T)$ as defined in Section~\ref{sec:submodular}.
By the total variance formula we get that
\begin{equation} \label{total variance formula}
\Var[\RReach_{\mathcal{G}}^\tau(T)] = {\Var}_{A}[\E[\RReach_{\mathcal{G}}^\tau(T | A)]] + \E_{A }[\Var[\RReach_{\mathcal{G}}^\tau(T | A)]].
\end{equation}
We bound the total variance by separately  bounding the two terms.

 \paragraph{Bound on the first term of the total variance:}
 We consider the reduced model
 $\mathcal{G}'$ with respect to $T$ and a
restriction of the influence function $\I^{\tau-1}_{\mathcal{G}'}$ to the domain that is subsets $A\subseteq N(T)$:
\[ f(A) := \I^{\tau-1}_{\mathcal{G}'}(A) .\]
 From Lemma~\ref{samedist}, this function represents the expected marginal utility value of nodes
 which are not in $T$ that are activated after $\tau$ steps if we
 activate $T$ at step $0$ and the set $A$ at step $1$.
 
 We first observe that $f$ is monotone and submodular. This because
strong monotone submodularity of our model implies that the 
 reduced model is also strongly monotone and submodular, and a restriction of
 a monotone and submodular function is also monotone and submodular.
 We establish two helpful properties of $f$. First,
\begin{equation} \label{prop2} 
f(\emptyset) = 0\ ,
\end{equation}
which holds for any influence function. Second, using Lemma~\ref{infrel:lemma} we obtain
\begin{equation} \label{prop3} 
\max_{v \in N(T)} f(v) \underbrace{\leq}_\text{\eqref{infrel}} \max_{v \in N(T)}  \I^{\tau-1}_{\mathcal{G}}(v)  \leq \M^{\tau-1}(\bar{T})\ .
\end{equation}

We are now ready to bound the first term of the total variance \eqref{total variance formula}.
Our monotone submodular function $f$ and the random subset
$A$ using probabilities $p_v$ satisfy
the conditions of Theorem~\ref{sub-additive_variable_variance_bound}. 
\begin{equation}\label{V<EM}
\begin{split}
& {\Var}_{A}[\E[\RReach^\tau(T | A)]] = {\Var}_{A}[f(A) + H(T)] = {\Var}_{A}[f(A)] \underbrace{\leq}_{\text{Theorem \;\ref{sub-additive_variable_variance_bound}}} (\max_{v \in N(T)} f(v)) \E_{A}\left[f(A) - f(\emptyset)\right] \\ 
& \underbrace{\leq}_\text{\eqref{prop2}, \eqref{prop3}} \M^{\tau-1}(\bar{T}) \E_{A}\left[ \E[\tReachonearg{\tau}{T | A} - H(T)]\right] = \M^{\tau-1}(\bar{T})\left(\I^\tau_{\mathcal{G}}(T) -H(T)\right)
\end{split}
\end{equation}


\paragraph{Bound on the second term of the total variance:}
We next bound the second term of \eqref{total variance formula} which is the expectation of the variance conditioned on $A$:

\begin{eqnarray}
    \E_{A}[\Var[\RReach_{\mathcal{G}}^\tau(T \mid A)]] &=& 
  \sum_{S \subset N(T)} \Pr[A=S]\Var[\RReach_{\mathcal{G}}^\tau(T \mid S)] \nonumber\\
  &=&  \sum_{S \subset N(T)}  \Pr[A=S] \E[\RReach_{\mathcal{G}}^\tau(T\mid S) -
      H(T)] \frac{\Var[\RReach_{\mathcal{G}}^\tau(T \mid S)]}{\E[\RReach_{\mathcal{G}}^\tau(T\mid
      S) - H(T)]}\nonumber\\
  &\leq& \max_{S \subset N(T)} \frac{\Var[\RReach_{\mathcal{G}}^{\tau}(T \mid
         S)]}{\E[\RReach_{\mathcal{G}}^{\tau}(T \mid S) - H(T)]} \sum_{S \subset
         N(T)}  \Pr[A=S] \E[\RReach_{\mathcal{G}}^{\tau}(T \mid S) - H(T)]  \nonumber\\
  &=& \E_{A}\left[ \E[\RReach_{\mathcal{G}}^{\tau}(T \mid A) - H(T)]\right] \max_{S
      \subset N(T)} \frac{\Var[\RReach_{\mathcal{G}}^{\tau}(T \mid
      S)]}{\E[\RReach_{\mathcal{G}}^{\tau}(T \mid S) - H(T)]}\nonumber\\
  &\underbrace{=}_{\eqref{aggA:eq}, \eqref{removing step},\eqref{varrel}}& \left(\I^\tau_{\mathcal{G}}(T) -H(T)\right) \max_{S \subset N(T)} \frac{\Var[\RReach^{\tau-1}_{\mathcal{G}'}(S)]}{\I[\RReach^{\tau-1}_{\mathcal{G}'}(S)]}\nonumber\\
    &=& 
  \left(\I^\tau_{\mathcal{G}}(T) -H(T)\right)
  \frac{\Var[\RReach^{\tau-1}_{\mathcal{G}'}(S')]}{\I[\RReach^{\tau-1}_{\mathcal{G}'}(S')]}\label{sbound}
\end{eqnarray}
Where we take
\[ S' = \arg\max_{S \subset N(T)}
\frac{\Var[\RReach^{\tau-1}_{\mathcal{G}'}(S)]}{\I[\RReach^{\tau-1}_{\mathcal{G}'}(S)]}\]
to be the subset that maximizes the ratio. 

Using the induction hypothesis on $(\tau-1)$-stepped influence we get
\begin{equation}\label{indh}
 \Var[\RReach^{\tau-1}_{\mathcal{G}'}(S')]\leq (\tau-1) \M^{\tau-2}_{\mathcal{G}'}(\bar{S'}) 
\I[\RReach^{\tau-1}_{\mathcal{G}'}(S')] .
\end{equation}

We now relate the maximum influence of nodes in the original and reduced models:
\begin{equation}\label{Mrel}
\M^{\tau-2}_{\mathcal{G}'}(\bar{S'})= \max_{v\in V\setminus(T\cup S')} \I^{\tau-2}_{\mathcal{G}'}(v) 
\underbrace{\leq}_\text{\eqref{infrel}}
\max_{v\in V\setminus(T\cup S')} \I^{\tau-2}_{\mathcal{G}}(v) \leq 
\M^{\tau-2}_{\mathcal{G}}(\overline{T\cup S'}) \leq \M^{\tau-1}_{\mathcal{G}}(\bar{T}) .
\end{equation}

From \eqref{sbound} using \eqref{indh} and \eqref{Mrel} we obtain
\begin{equation}\label{secondpart}
\E_{A}[\Var[\tReachonearg{\tau}{T | A}]]\underbrace{\leq}_\text{\eqref{indh}}
\left(\I^\tau_{\mathcal{G}}(T) -H(T)\right)(\tau-1)\M^{\tau-2}_{\mathcal{G}'}(\bar{S'}) \underbrace{\leq}_\text{\eqref{Mrel}} \left(\I^\tau_{\mathcal{G}}(T) -H(T)\right)(\tau-1)\M^{\tau-1}_{\mathcal{G}}(\bar{T})\  . 
\end{equation}

\paragraph*{Combining the bounds of the first and second terms}
The claim of the Theorem follows using total variance \eqref{total variance formula} and the bounds on the first term \eqref{V<EM} and second term \eqref{secondpart}.

\section{Variance lower bound construction}~\label{varLB:sec}
\begin{lemma}\label{high of a tree lemma}
Let $\mathcal{G}$ be complete binary tree where each edge has probability $\frac{1}{2}$ and let $h(u)$ be the height of the node $u$. Then, $\I^\tau(u) = h(u)$ and $\Var[\tReachonearg{\tau}{u}] = \frac{1}{2} \sum_{i=0}^{h(u)-1} i^2 = \frac{\left(h(u)-1\right)h(u)\left(2h(u)-1\right)}{12}$.
\end{lemma}

\begin{proof}

By induction on the height of the of the node.

\textbf{base step}: ($h(u)=1$):
It is clear that $\I^1(u) = 1$ and $\Var\left[\tReachonearg{1}{u}\right] = 0$ since $u$ is a leaf.

\textbf{Inductive step}:
$u$ has two neighbors and each is reached with probability $\frac{1}{2}$. Let $v_1$ and $v_2$ be the neighbors of $u$ and let $X_1, X_2$ be random variables that indicate if $(u,v_1), (u,v_2)$ were activated respectively. The variables $X_1, X_2$ are Bernoulli random variables with $p=\frac{1}{2}$, hence, $\E[X_1] = \E[X_2] = \frac{1}{2}$ and $\Var[X_1] = \Var[X_2] = \frac{1}{4}$. Since the graph is a tree, the reachabilities of $v_1$ and $v_2$ are independent random variables, so we can simply write:

\[ \tReachonearg{\tau}{u} = 1 + X_1\tReachonearg{\tau-1}{v_1} + X_2 \tReachonearg{\tau-1}{v_2} \]

The variable $\tReachonearg{\tau-1}{v_1}$ and $\tReachonearg{\tau-1}{v_2}$ are identical and $X_1, \tReachonearg{\tau-1}{v_1}$ and $X_2, \tReachonearg{\tau-1}{v_2}$ are independent random variables, Thus,

\[ \I^\tau(u) = \E[\tReachonearg{\tau}{u}] = 1 + \E[X_1]\E[\tReachonearg{\tau-1}{v_1}] + \E[X_2]\E[\tReachonearg{\tau-1}{v_2}] = 1 + 2 \frac{1}{2}\E[\tReachonearg{\tau-1}{v_1}]\] 
\[\underbrace{=}_\text{induction's hypothesis} 1 + h(u)-1 = h(u) \]

The computation of the variance is similar:

\[ \Var[\tReachonearg{\tau}{u}] = \Var[X_1\tReachonearg{\tau-1}{v_1}] + \Var[X_2\tReachonearg{\tau-1}{v_2}] = 2\Var[X_1\tReachonearg{\tau-1}{v_1}] \]

For two independent random variables $A, B$ holds that: $\Var[AB] = \Var[A]\Var[B] + \Var[A]\E^2[A] + \Var[B]\E^2[B]$, we have that:

\[ \Var[\tReachonearg{\tau}{u}] = 2 \left(\Var[X_1] \Var[\tReachonearg{\tau-1}{v_1}] + \Var[X_1] \E^2[\tReachonearg{\tau-1}{v_1}] + \Var[\tReachonearg{\tau-1}{v_1}] \E^2[X_1]\right)\]

\[ = 2\left( \frac{1}{2}\Var[\tReachonearg{\tau-1}{v_1}] + \frac{1}{4}\E^2[\tReachonearg{\tau-1}{v_1}] \right) = \Var[\tReachonearg{\tau-1}{v_1}] + \frac{1}{2}(h(u)-1)^2\]
\[\underbrace{=}_\text{induction's hypothesis} \frac{1}{2}\sum_{i=0}^{h(u)-1} i^2 = \frac{(h(u)-1)h(u)(2h(u)-1)}{12}.\]

\end{proof}

\begin{theorem}
There is a model $\mathcal{G}$ and a set of nodes $T$ such that $\frac{\Var[\tReachonearg{\tau}{T}]}{\M^\tau(\bar{T})\I^\tau(T)} \geq \frac{\tau}{12}$.  
\end{theorem}

\begin{proof}

Lemma \ref{high of a tree lemma} shows that for every node $u \in \mathcal{G}$, $\I^\tau(u) = h(u)$ and $\Var[\tReachonearg{\tau}{u}] = \frac{\left(h(u)-1\right)h(u)\left(2h(u)-1\right)}{12}$. It follows that the root $r$ has the largest influence $\I^\tau(r) = \tau$ and $\Var[\tReachonearg{\tau}{u}] = \frac{(\tau-1)t(2\tau-1)}{12}$, Furthermore $\M^\tau(\bar{r}) = \tau-1$ since the nodes of the largest influence in $V \setminus r$ are the children of $r$. We conclude that:

\[\frac{\Var[\tReachonearg{\tau}{r}]}{\M^\tau(\bar{r})\I^\tau(r)} = \frac{(\tau-1)\tau(2\tau-1)}{\tau(\tau-1)12} = \frac{2\tau-1}{12} \underbrace{\geq}_{\tau \geq 1} \frac{\tau}{12}.\]

\end{proof}

\section{Greedy optimization with approximate non-submodular oracle} \label{greedyproof:sec}
In this section we
 present the proof of Lemma~\ref{almostsubmodular:thm}.
 We show that our approximation guarantees imply that the application of greedy on $\hat{F}$ generates a sequence that is an approximate greedy sequence (in the sense of Lemma~\ref{approxgreedy:lemma}) with respect to $F$.
 
 We first state a helpful Lemma \cite{binaryinfluence:CIKM2014} that establishes that it suffices that $\hat{F}(u \mid S)$ to approximate the marginal contributions \[F(u \mid S) := F(S\cup\{u\})-F(S)\].  
  \begin{lemma} \cite{binaryinfluence:CIKM2014}\label{approxgreedy:lemma}
  Given a monotone submodular function $F$, an
  approximate greedy algorithm that for some $\epsilon\in[0,1)$ selects at each step an element 
  $u$ such that 
  $F(u \mid S) \geq (1-\epsilon)\max_v F(v \mid S)$ has approximation ratio
  $\geq (1-(1-1/s)^s)(1-\epsilon)$.
  \end{lemma}
  \begin{proof}
  It is easy to see that the approximation ratio of $\epsilon$-approximate greedy is
  $1-(1-(1-\epsilon)/s)^s$. It therefore suffices to establish that
  this expression is larger than $(1-(1-1/s)^s)(1-\epsilon)$ for $\epsilon\in [0,1].$
  Equivalently, we need to show that for all $s\geq 2$ and $x\in[0,1]$ 
  $$(1-(1-x)/s)^s - (1-x)(1-1/s)^s -x \leq 0\ .$$  This follows from equality holding for
  $x=0$ and $x=1$ and the function being concave up (second derivative is positive).  
    \end{proof}

 \begin{proof} [Proof of Lemma~\ref{almostsubmodular:thm}]
 
 Consider a monotone non-negative $\hat{F}$ that is a uniform $\epsilon_A$-approximation of a monotone non-negative $F$ with $\epsilon_A= \frac{\epsilon(1-\epsilon)}{14s}$.
 By definition of $\epsilon$-approximation
 (see Section \ref{sec:influence-oracle}), 
 $\left| \hat{F} (T) - F(T)
      \right| \geq \epsilon_A  \max\{F(T),\OPT_1(F) \}$ for all $S$ with $|S|\leq s$.
Therefore, 
    \begin{align}
 \text{if } F(S) \geq (1-\epsilon)\OPT_1(F)\; \text{then} \; &  \frac{\left|\hat{F}(S)-F(S)\right|}{F(S)} \leq \frac{\epsilon}{14s} \label{eq:relative-cond}\\
 \text{if } F(S)\leq (1-\epsilon)\OPT_1(F) \; \text{then} \; & \hat{F}(S) \leq \left(1-\frac{\epsilon}{2}\right) \OPT_1(F)\ . \label{eq:abs-error}
\end{align}
Inequality~\eqref{eq:relative-cond} follows immediately when $F(S)\geq \OPT_1(F)$ because the relative error is at most $\epsilon_A \leq \frac{\epsilon}{14s}$. For $(1-\epsilon)\OPT_1(F) \leq F(S) < \OPT_1(F)$ we have absolute error being at most $\epsilon_A \OPT_1(S)$ which is a relative error of at most $\epsilon_A/(1-\epsilon)\leq \frac{\epsilon}{14s}$. 
Inequality~\eqref{eq:abs-error} follows from the absolute error being at most $\epsilon_A \OPT_1(F)$ and $\epsilon_A \leq \epsilon/2$.

We establish that these conditions imply that greedy on $\hat{F}$ 
on the prefix of the greedy sequence where $F(S)\leq \frac{3}{4}\OPT_s(F)$ is actually
approximate greedy (as in the conditions of 
Lemma~\ref{approxgreedy:lemma})
with respect to $F$.
Note that $1-(1-1/s)^s \ge 3/4$ for $s\ge 2$ and thus the prefix restriction does not limit generality.
The claim will then follow from
 Lemma~\ref{approxgreedy:lemma}.

 For $s=1$, it follows from Equations \eqref{eq:relative-cond} and \eqref{eq:abs-error} , that the first element of a greedy sequence with respect to $\hat{F}$, $\arg\max_u \hat{F}(u)$, satisfies  $\hat{F}(u) \geq (1-\frac{\epsilon}{14}) \OPT_1(F)$.
 Therefore from
 the second iteration and on, we have a set $S$ for which the relative error bounds in Equation \eqref{eq:relative-cond}) applies.
 
 We consider the marginal contributions $F(u \mid S)$ for any node $u$.  We have
 \begin{eqnarray}
 |\hat{F}(S)-F(S)| &\leq& \frac{\epsilon}{14s}F(S) \nonumber\\
 |\hat{F}(S \cup\{u\})-F(S\cup\{u\})| &\leq& \frac{\epsilon}{14s}F(S\cup\{u\})= \frac{\epsilon}{14s}(F(S)+F(u \mid S))\ .\nonumber
 \end{eqnarray}
 We use these inequalities to bound the absolute error of (any) marginal influence estimate by
 \begin{eqnarray} 
 \left| \hat{F}(u \mid S) - F(u \mid S) \right| &=&
 \left|\hat{F}(S \cup \{u\})-\hat{F}(S)- F(S\cup \{u\})+ F(S) \right|   \label{absmarg:eq} \\
 &\leq&  \left|\hat{F}(S \cup \{u\})-F(S\cup\{u\}\right| + \left| \hat{F}(S) -F(S) \right| \nonumber\\
 &\leq& \frac{\epsilon}{7s}F(S)+ \frac{\epsilon}{14s}F(u \mid S)\ .\nonumber
 \end{eqnarray}
 
We now consider the node $v=\arg\max_{u\in V} F(u \mid S)$ with
maximum marginal contribution to $S$ with respect to $F$ and its
contribution value 
\[\Delta := F(v \mid S) \geq \frac{1}{s}(\OPT_s -F(s))\ .\]
Thus, when 
$F(S)\leq \frac{3}{4}\OPT_s$, 
\begin{equation} \label{deltabound:eq}
\Delta \geq \frac{1}{3s}F(S) .
\end{equation}

By applying (\ref{absmarg:eq})
to $v$ we get that 
\[\hat{F}(v \mid S) \geq \Delta-\frac{\epsilon}{7s}F(S)- \frac{\epsilon}{14s}\Delta\]

Therefore the node 
$v' =\arg\max_v \hat{F}(v' \mid S)$ with
maximum marginal contribution according to $\hat{F}$ satisfies 
 \[\hat{F}(v' \mid S) \geq \hat{F}(v \mid S) \geq \Delta-\frac{\epsilon}{7s}F(S)- \frac{\epsilon}{14s}\Delta. \]
 
By using (\ref{absmarg:eq}) again,  
 substituting \eqref{deltabound:eq}, and using that fact that $s\geq 2$:
\begin{eqnarray}
  F(v' \mid S) &\geq& \Delta-2\frac{\epsilon}{7s}F(S)- 2\frac{\epsilon}{14s}\Delta\\
  &\geq& \Delta-\epsilon\Delta (\frac{1}{7s}+\frac{6}{7} ) \geq \Delta \left(1- \epsilon \right)\ .
\end{eqnarray}

Therefore, the greedy sequence according to $\hat{F}$ is an approximate greedy sequence according to $F$ and satisfies the conditions of Lemma~\ref{approxgreedy:lemma}. Therefore the resulting sequence yields an approximation ratio at least
$(1-(1-1/s)^s)(1 - \epsilon)$.
  \end{proof}

  \section{Greedy for live-edge models} \label{greedymoa:sec}

  Proof of Theorem~\ref{greedyalg:thm}:
   \begin{proof}
   For the first bound, we explicitly maintain for each node $u\in V$, for each pool, the reachability set of $u$ in the simulations of the pool (and its cardinality). 
   The dominant term in the cost of computation is performing a BFS from each node in each of the $r \ell $ simulations that is truncated at distance $\tau$. 
   The total computation time is 
   \begin{equation}
   O(r\ell \overline{m}n)=O(\epsilon^{-2} s^3 c \ln 
   \left(\frac{n}{\delta}\right) \overline{m}n)\ , 
   \end{equation}
   where $$\overline{m} = \frac{1}{\ell r} \sum_{i=1}^r\sum_{j=1}^\ell |E_{ij}|$$ is the average number of edges per simulation. For an IC model, $\E[\overline{m}]=\sum_{e \in \mathcal{E}} p_e$. 
   When a node $u$ is selected into the seed set we 
   remove all nodes in its reachability set from 
   the reachability sets of all other nodes. The removal cost can be "charged" to the initial reachability computation.  
   
    The dependence of the computation time on the graph size can be improved by using combined reachability sketches~\cite{ECohen6f,binaryinfluence:CIKM2014,timedinfluence:2015,ECohenADS:TKDE2015} instead of maintaining the reachability sets explicitly (see  Section~\ref{sketchedave:sec}).  The sketch size needed in order to provide the required accuracy of $O(\epsilon/s)$ (as in Theorem~\ref{almostsubmodular:thm}) uniformly for all subsets of size at most $s$ is $k=O(\epsilon^{-2}s^3\ln{n})$.  We compute a sketch for each node in each of the $r$ pools, so in total we have $rn$ node sketches. The construction time of these sketches has a term $\sum_{ij} |E_{ij}|=r \ell \overline{m}$ linear in the total size of simulations and a term for sketch constructions which is a product of 
    the number of pools $r$ and the construction time for each pool. The per-pool construction time is as described in Section~\ref{sketchedave:sec} and is bounded by $k$ (sketch size) visits for each node, each involving reverse traversals of incoming edges of the node in some simulation.  The per-pool construction time for an IC model is  $O(k(n+\sum_e p_e))$ in expectation.  The time with arbitrary simulations for pool $i$ is $O(k(n+\sum_v \max_j d_v(E_{ij})))$.  In total over all pools, the construction time is dominated by 
    $O(r (\ell\overline{m} + k (n+ m^*))$, where $m^* = \sum_v \max_{ij} d_v(E_{ij})$ for arbitrary simulations and $m^*=\sum_e p_e$ for simulations generated by an IC model. 
    
    The sketches improve the computation time of greedy.  Each iteration of greedy uses the (precomputed) union sketch of the current seed set $S$ in each pool. It then examines the sketches of each $v\in V$ to compute the estimate of the averaging oracle  $\eA(S\cup \{v\})$ in each pool.  This operations takes $O(k n r)$ in total for the iteration. Therefore $s$ iterations of greedy maximization takes $O(knrs)$ using the sketches. Combining the construction cost of the sketches using $r=O(s\ln{\frac{n}{\delta}})$ and $\ell = O(\epsilon^{-2}s^2 c)$ and the greedy implementation over the sketches we obtain a total bound on the computation time of 
    \begin{eqnarray*}
    O(r\ell\overline{m} + k r (m^*+ns))&=& O(r(\ell\overline{m} +k(m^*+ns)))\\ &=&  O(s\ln\frac{n}{\delta}(\epsilon^{-2}s^2 c \overline{m}+ \epsilon^{-2}s^3 (m^* + ns) \ln n)\\ &=& O(\epsilon^{-2}s^3\ln\frac{n}{\delta}(c \overline{m} +s(m^*+ns)\ln n))\ . 
    \end{eqnarray*}
    \end{proof}
  
\section{Optimization with adaptive sample size} \label{adaptivemore:sec}

The pseudocode for our wrapper is provided in Algorithm~\ref{optsamples:alg}.  The inputs to the wrapper are a base algorithm ${\mathcal A}$ and two constructions of oracles from sets of
simulations.  The first construction produces an oracle,  $\hat{F_v}$, that we use for
validation.
The second construction produces oracles, $\hat{F}_x$, that are provided as input to ${\mathcal
  A}$ to perform the optimization.  The oracles provide an
approximation of our influence function $\I^\tau(S)$ with non-uniform
guarantees.  For specified $(\epsilon,\delta)$ we use the  expressions
 $r_v(\epsilon,\delta)$ or
$r_x(\epsilon,\delta)$ for the number of simulations required to obtain $(\epsilon,\delta)$ guarantees (in the sense of Section~\ref{sec:influence-oracle}).
This gives us a relation between $\epsilon$, $\delta$, and a number of simulations. 
When constructing an oracle with a given number of simulations $r$ and a specified $\epsilon$, we can determine the confidence $\delta$ we have from $\epsilon$ and $r$.  The oracles that we consider
 have the property that for a fixed $\epsilon$,  $\delta$ decreases at least linearly with the number of simulations. 
 \notinproc{(i.e., when we double the number of simulations $\delta$ decreases by at least a factor of $2$.)}

 \begin{algorithm}[htbp]\caption{Optimization wrapper \label{optsamples:alg}}
  \DontPrintSemicolon
  \KwIn{(i) Two oracle constructions from simulations: 
    $\hat{F_{v}}$ (validation) and $\hat{F_{x}}$
    (optimization) that with
    $r_v(\epsilon,\delta)$ (resp., $r_x(\epsilon,\delta)$) simulations provide
 $(\epsilon,\delta)$ guarantees.
 (ii)  An optimization
      algorithm ${\cal A}$ that applies to the optimization oracle
      $\hat{F_x}$ and returns a subset. (iii)  $M$: Bound on maximum number of
      simulations. (iv) Parameters $\epsilon>0$ and $\delta>0$.}
  {\small
 $r\gets r_x(\epsilon,\delta)$ \tcp*{\#simulations for
   $\hat{F_{x}}$ to provide $(\epsilon,\delta)$ guarantees}
 \tcp{Build validation oracle}
 $\delta_v \gets \frac{\delta}{\lceil \log_2 M/r) \rceil}$;
 $r_v \gets r_v(\epsilon, \delta_v)$ \tcp*{\#simulations for validation oracle}\;
  $\hat{F_{v}} \gets$ validation oracle from $r_v$
  i.i.d simulations that provides $(\epsilon,\delta_v)$ guarantees\;
  
  ${\mathcal R} \gets \perp$ \tcp*{Initialize set of i.i.d simulations for
  optimization}\;
  \Repeat{$|{\mathcal R}|+r_v>M$}{
    Add $r$ fresh i.i.d simulations to set ${\mathcal R}$\;
    
   $\hat{F}_{x}\gets $ optimization oracle from simulations
   ${\mathcal R}$ that provides $(\epsilon,*)$ guarantees 
   \tcp*{* determined by $|\mathcal{R}|$}\;
    $T \gets {\cal A} (\hat{F_x})$\tcp*{Optimize over the oracle}
    \eIf{$\hat{F_v}(T)\geq \frac{(1-2\epsilon)}{1+\epsilon} 
      \hat{F_x}(T)$}{\Return{$T$, $\hat{F_v}(T)$}}
{ $r \gets 2r$}
}
}
  \end{algorithm}
     
The wrapper first determines an upper bound ($\lceil \log_2 M/r_x(\epsilon,\delta) \rceil$) on the maximum number of iterations it performs (based on the initial number and the simulation budget) and constructs a validation oracle that provides guarantees
for a small number of sets (queries) which equals this maximum number of
iterations.    It then starts with a set $\mathcal{R}$ of $r_x(\epsilon,\delta)$ simulations that
suffice for the oracle $\hat{F}_x$
to provide (non-uniform)  $(\epsilon,\delta)$ approximation guarantees.  The wrapper repeats
the following:  It constructs an ``optimization'' oracle
$\hat{F}_x$ using the set of
simulations $\mathcal{R}$ and applies ${\mathcal A}$ over $\hat{F}_x$ to obtain
a set $T$. The wrapper terminates when $\hat{F}_v(T)$ is close to  $\hat{F}_x(T)$
or when our simulation budget of $M$ is exceeded. Otherwise, we double the
number of simulations in our set $\mathcal{R}$ and repeat.

The wrapped algorithm ${\mathcal A}$ can be an exact or approximate optimizer.  It is applied to the oracle function and therefore its quality guarantees are with respect to how well the oracle value $\hat{F}_x(T)$ of the output set $T$ approximates the oracle optimum $\max_{S \mid |S|\leq s}\hat{F}_x(S)$.
  The wrapper extends the approximation guarantees that ${\mathcal A}$
  provides (with respect to the oracle) to a guarantee with respect to 
  the influence function while avoiding the worst-case number of
  simulations needed for a uniform approximation.

  We first establish some basic properties.
  \begin{lemma} \label{onesided:lemma}
Let $S$ be a set with maximum influence (with $\I^\tau(S)=
\OPT^\tau_s$).    
  With probability at least
  $1-2\delta$, all the optimization oracles
  $\hat{F_x}$ constructed by the
  wrapper have 
  $ (1-\epsilon)\OPT^\tau_s \le \hat{F_x}(S) \le (1+\epsilon)
\OPT^\tau_s$.  
\end{lemma}
\begin{proof}
 The probability that 
  $ (1-\epsilon)\OPT^\tau_s \le \hat{F_x}(S) \le (1+\epsilon)
\OPT^\tau_s$ fails
  for 
the first oracle is at most $\delta$.  The number of $\hat{F}_x$ uses simulations doubles in each iteration 
and all our constructions are such that the 
confidence parameter $\delta$ decreases at least linearly with the 
number of simulations.  We therefore obtain that the sequence of
failure probabilities for $(1-\epsilon)\OPT^\tau_s \le \hat{F_x}(S) \le (1+\epsilon)
\OPT^\tau_s$
is geometric and sums up to at most $2\delta$. 
\end{proof}
As an immediate corollary we obtain:
\begin{corollary} \label{onesided:coro}
Under the conditions of Lemma~\ref{onesided:lemma}, 
the oracle optimum in all iterations satisfies
\[\max_{T \mid |T|\leq s}\hat{F_x}(T) \geq (1-\epsilon)
\OPT^\tau_s\ .\]
\end{corollary}

The following is immediate  from the
construction of the validation oracle.
\begin{lemma} \label{validation:lemma}
With probability at least $1-\delta$, the validation oracle has
relative error at most $\epsilon$ on all tests in which  the input set $T$ is such that 
$\I^\tau(T) \geq \OPT^\tau_1$ and absolute error at most
$\epsilon \OPT^\tau_1$ otherwise.
\end{lemma}  
\begin{proof}
The wrapper performs at most $\lceil \log_2 M/r \rceil$ iterations before it stops, in each iteration the validation oracle fails to provide an $\epsilon$-approximation with probability at most $\delta_v$. Therefore, by union bound, the probability that the algorithm fails to provide an $\epsilon$-approximation in at least one round is at most $\delta_v \lceil \log_2 M/r \rceil \leq \delta$.
\end{proof}

\begin{lemma}\label{approximate maximizer:lemma}
Assume that our data and our optimization oracle  with $r$ or more 
simulations, are such that with probability at least $1-\delta$,  the optimum of the oracle is an approximate
optimizer, that is:
\begin{equation}\label{eq:first_cond}
  (1+\epsilon)\OPT_s^\tau \geq  \max_{S \mid |S| \leq
  s}\hat{F_x}(S) \geq  (1-\epsilon)\OPT_s^\tau
\end{equation}

\begin{equation}\label{eq:second_cond}
\I^\tau(\arg\max_{S \mid |S|\leq s}\hat{F_x}(S)) \geq
   (1-\epsilon)\OPT_s^\tau    
\end{equation}
and assume that the algorithm ${\cal A}$ returns the oracle optimum. Then
with probability at least $1-5\delta$, the wrapper terminates after at most
$2\max\{r,r_x(\epsilon,\delta)\}+r_v$ simulations and returns $T$ such that
$\I^\tau(T) \geq (1-5\epsilon) \OPT_s^\tau$.
\end{lemma}
\begin{proof}
First we show that the wrapper returns a set $T$ with the required properties with probability at most $1 - 3\delta$ and then we show that the number of iterations the wrapper does before it stops is smaller than $M$ with probability of at most $1 - 2\delta$.

From Lemma~\ref{onesided:lemma}, with probability at least $1-2\delta$ in all iterations ${\mathcal A}$ returns $T$ for which
$\hat{F_x}(T) \geq (1-\epsilon)\OPT^\tau_s$.   
The validation succeeds only if
$\hat{F_v}(T) \geq \frac{(1-2\epsilon)}{1+\epsilon} \hat{F_x}(T)  \geq
\frac{(1-\epsilon)(1-2\epsilon)}{1+\epsilon} \OPT^\tau_s$.
From Lemma~\ref{validation:lemma} with probability at least
$1-\delta$ in all iterations  we have
\[\hat{F_v}(T) \leq \max\{(1+\epsilon) \I^\tau(T),  \I^\tau(T) +
\epsilon \OPT^\tau_1\}\ .\]


Therefore, with probability $1 - 3\delta$ the set $T$ returned by the wrapper satisfies 
\[\frac{(1-\epsilon)(1-2\epsilon)}{1+\epsilon} \OPT^\tau_s \leq \max\{(1+\epsilon) \I^\tau(T),  \I^\tau(T) + \epsilon \OPT^\tau_1 \}\ . \]

If $(1+\epsilon)\I^\tau(T) > \I^\tau(T) + \epsilon \OPT^\tau_1$ then $\I^\tau(T) \geq \frac{(1-\epsilon)(1-2\epsilon)}{(1 + \epsilon)^2} \OPT^\tau_s \geq (1-5\epsilon)\OPT^\tau_s$. Otherwise, we have that $\I^\tau(T) \geq \left( \frac{(1-\epsilon)(1-2\epsilon)}{1+\epsilon} - \epsilon \right) \OPT^\tau_s \geq (1-5\epsilon)\OPT^\tau_s .$

We have to show that with probability at least $1-2\delta$ within $2\max\{r,r_x(\epsilon,\delta)\}+r_v$ simulations the wrapper returns such a set $T$ to finish the proof.
Consider the first iteration where $|{\mathcal R}|\geq r$. By Equations (\ref{eq:first_cond}) and (\ref{eq:second_cond}) with
  probability at least $1-\delta$ we have that $\I^\tau(T)\geq
  (1-\epsilon)  \OPT_s^\tau$ and
$(1+\epsilon) \OPT_s^\tau \geq \hat{F_x}(T) \geq (1-\epsilon)\OPT_s^\tau$ .By Lemma \ref{validation:lemma} we have that with probability at least $1-\delta$, the validation oracle satisfies that $\hat{F_v}(T) \geq (1-\epsilon) \I^\tau(T)$ or $\hat{F_v}(T) \geq \I^\tau(T) - \epsilon \OPT_s^1 \geq \I^\tau(T) - \epsilon \OPT_s^\tau$. By the last two statements we have that with probability of at least $1-2\delta$:

\[\hat{F_v}(T) \geq (1-\epsilon) \I^\tau(T) \geq (1-\epsilon)^2
\OPT_s^\tau \geq (1-2\epsilon) \OPT_s^\tau \geq \frac{(1-2\epsilon)}{1+\epsilon} \hat{F}_x(T) \]

or 

\[\hat{F_v}(T) \geq \I^\tau(T) - \epsilon \OPT_s^\tau  \geq (1-2\epsilon) \OPT_s^\tau \geq \frac{(1-2\epsilon)}{1+\epsilon} \hat{F}_x(T) . \]


\end{proof}  

Theorem~\ref{optAadaptive:thm}, which we restate below to provide reading fluency,  now follows as a corollary. 
\begin{theorem} [Theorem~\ref{optAadaptive:thm}]
Suppose that on our data the averaging (respectively, median-of-averages) oracle
$\hat{F}$ has the
property that with $r$ simulations, with probability at least $1-\delta$,  the oracle optimum
$T := \arg\max_{S \mid |S|\leq s}\hat{F}(S)$
satisfies
\begin{eqnarray*}
\I^\tau(T) &\geq& (1-\epsilon)\OPT_s^\tau .\\
\end{eqnarray*}
Then with probability at least
$1-5\delta$, when using
$2\max\{r,r(\epsilon,\delta)\} + O\left(\epsilon^{-2}c\left(\ln{\frac{1}{\delta}} +   \ln \left(\ln\ln \frac{n}{\delta}+ \ln s\right)\right)\right)$ simulations with the
median-of-averages oracle and
$2\max\{r,r(\epsilon,\delta)\} + O\left(\epsilon^{-2}c\left(\ln{\frac{1}{\delta}} + \ln \left( \ln\ln \frac{n}{\delta}+ \ln n \right) \right)\right)$ simulations with the averaging
oracle, the wrapper outputs a
set $T$ such that $\I^\tau(T) \geq (1-5\epsilon)\OPT_s^\tau$.
\end{theorem}

\begin{proof} [Proof of Theorem~\ref{optAadaptive:thm}]
We analyze here the number of simulations required using the averaging oracles and the median-of-averages oracles, in both cases we use median-of-averages oracles for validation. In both cases $r_v = r(\epsilon, \delta_v) = O\left(\epsilon^{-2}c \log{\frac{1}{\delta_v}} \right)$, where $\delta_v = \frac{\delta}{\lceil\log_2{ \frac{M}{r_x} }\rceil}$. By Lemma \ref{approximate maximizer:lemma} the number of simulations is at most $2\max\{r,r_x(\epsilon,\delta)\}+r_v$. $M$ and $r_x$ get different values for each oracle.

\textbf{median-of-averages oracles analysis}

We have that $r_x = O(\epsilon^{-2}c \ln{\delta^{-1})}$ by Lemma \ref{MEoracle:lemma} and we set $M = O(\epsilon^{-2}c s \ln{\frac{n}{\delta}})$ by Theorem~\ref{simupper:thm}. Simple calculation shows that:

\[ \frac{1}{\delta_v} = \frac{\lceil \ln{\frac{s\ln{\frac{n}{\delta}}}{\ln{\frac{1}{\delta}}}} \rceil}{\delta} \leq \frac{1}{\delta} \left( \ln{s} +  \ln{(\ln{\frac{n}{\delta}})}\right) .\]

Therefore, 
\[r_v = O\left(\epsilon^{-2}c\left(\ln{\frac{1}{\delta}} +   \ln \left(\ln\ln \frac{n}{\delta}+ \ln s\right)\right)\right) .\]

\textbf{averaging oracles analysis}

We have $r_x = O(\epsilon^{-2}c\delta^{-1})$ and we set $M = O \left( \epsilon^{-2} s n \ln{\frac{n}{\delta}} \right)$ according to the respective worst-case guarantees on the number of simulations specified in \eqref{naive:eq}. A simple calculations shows:

\[ \frac{1}{\delta_v} = \frac{\lceil \ln{\frac{\delta s n \ln{\frac{n}{\delta}}}{c}}\rceil }{\delta} \leq \frac{1}{\delta}\left( \ln\ln{\frac{n}{\delta}} + 2\ln{n} \right) \] 

Therefore,
\[r_v = O\left(\epsilon^{-2}c\left(\ln{\frac{1}{\delta}} + \ln \left( \ln\ln \frac{n}{\delta}+ \ln n \right) \right)\right) .\]

\end{proof}

  We next consider cases where the algorithm ${\mathcal A}$ is approximate (may not return the oracle optimizer). We assume in these cases that the optimization oracles $\hat{F}_x$ when constructed with a given number of simulations provide, with high probability,  uniform $\epsilon$-approximation for all $\binom{n}{s}$ subsets of cardinality at most $s$:
  \begin{equation*} \forall T \text{ such that } |T|<s,\ 
\left| \hat{F}(T) - \I^\tau(T) \right| \leq
                      \epsilon \max\{\I^\tau(T),\OPT^\tau_1\} \ .
  \end{equation*}

  We first show that a very weak assumption on $\mathcal{A}$ suffices to guarantee termination with good probability.
 \begin{lemma} \label{termination:lemma}
If the optimization oracle $\hat{F_x}$ when constructed with $r$ or more
simulations provides uniform $\epsilon$-approximation with probability at 
least $1-\delta$, and the algorithm
   ${\cal A}$ returns $T$
such that $\hat{F_x}(T) \geq
(1 - \epsilon)\OPT_1^\tau$. 
Then with probability at least $1-2\delta$ the 
wrapper will terminate after using at most $2\max\{r,r_x(\epsilon,\delta)\}+r_v$ simulations.
\end{lemma}
\begin{proof}
Consider the first iteration where $\hat{F_x}$ is constructed
using at least $r$ simulations. Let $T$ be the set that ${\mathcal A}$ returns at this iteration.
Since $\hat{F}_x$ provides uniform $\epsilon$-approximation we have that $\hat{F_x}(T)\leq
(1+\epsilon) \I^\tau(T)$ with probability at least $1 - \delta$.  Combining this with our assumption we get that $\I^{\tau}(T) \geq \OPT^{\tau}_1$ and by Lemma $\ref{validation:lemma}$ we have that with probability at least $1 - \delta$ if $\I^{\tau}(T) \geq \OPT_1^{\tau}$ then $\hat{F}(T) \geq (1 - \epsilon)\I^{\tau}(T)$ and if $\frac{1-\epsilon}{1+\epsilon}\OPT^{\tau}_1 \leq \I^\tau(T) \leq \OPT^{\tau}_1$ then $\hat{F}_v(T) \geq \I^\tau(T) - \epsilon \OPT^{\tau}_1 \geq \I^{\tau}(T) - \frac{\epsilon(1 + \epsilon)}{1 - \epsilon}\I^{\tau}(T)$. Combining we obtain that $\hat{F}_v(T) \geq \frac{1-2\epsilon}{1+\epsilon}\hat{F}_x(T)$, and thus the validation condition holds.
\end{proof}  

We next consider algorithms ${\cal A}$ that guarantees some approximation ratio $\rho$.
  \begin{theorem}\label{weak alg:theorem}
  Suppose that our optimization oracle when constructed with $r$ or more simulations provides uniform $\epsilon$-approximation with probability at least $1-\delta$. Assume now that the algorithm ${\cal A}$ returns a set $T$
such that 
\[\hat{F_x}(T) \geq \rho \max_{S \mid |S|\leq s}
\hat{F_x}(S)\ .\]
Then,
the set $T$ returned by our wrapper satisfies
$\I^\tau(T) \geq \rho(1-5\epsilon)\OPT^\tau_s$ with probability of at least $(1-3\delta)$. 
\end{theorem}
\begin{proof}
Consider an optimal set $S$ (with $\I^\tau(S)= \OPT^\tau_s$). By Lemma \ref{onesided:lemma} with probability at least $1-2\delta$ all
our oracles have $(1 - \epsilon)\OPT^{\tau}_s \leq \hat{F_x}(S) \leq (1 + \epsilon)\OPT^{\tau}_s $ are within $(1\pm
\epsilon)\OPT^\tau_s$. By the assumption, the sets $T$ returned by ${\cal A}$ in all iterations have
$\hat{F_x}(T) \geq \rho \max_{S \mid |S|\leq s}
\hat{F_x}(S) \geq \rho (1-\epsilon) \OPT^\tau_s$. When the wrapper stops we have that $\hat{F_x}(T) \leq \frac{1 + \epsilon}{1-2\epsilon} \hat{F_v}(T)$ and by Lemma \ref{validation:lemma} we have with probability at least $1-\delta$ that $\hat{F_v}(T) \leq \max\{ (1+\epsilon)\I^\tau(T), \I^\tau(T) + \epsilon \OPT^\tau_1 \}$.

Combining, we have that with probability at least $1 - 3\delta$,

\[ \rho(1-\epsilon)\OPT_s^\tau \leq \rho \hat{F_x}(S) \leq \hat{F_x}(T) \leq \frac{1+\epsilon}{1 - 2\epsilon} \hat{F_v}(T) \leq \frac{1+\epsilon}{1 - 2\epsilon}\max\{ (1 - \epsilon)\I^\tau(T), \I^\tau(T) + \epsilon \OPT^\tau_s \} .\]

Now, a simple calculation shows that $\I^\tau(T) \geq \rho(1 - 5\epsilon)\OPT^\tau_s$.

\end{proof}

We can prove now Theorem~\ref{greedyadaptive:thm} (restated for reading fluency):
\begin{theorem} [Theorem~\ref{greedyadaptive:thm}]
If the averaging oracle $\eA$  has the
property that with $\geq r$ simulations, with probability at least $1-\delta$,  it provides a uniform  $\epsilon$-approximation for all subsets of size at most $s$,  then with
$2\max\{r,r(\epsilon,\delta)\} + O\left(\epsilon^{-2}c\left(\ln{\frac{1}{\delta}} + \ln \left( \ln\ln \frac{n}{\delta}+ \ln n \right) \right)\right)$ simulations we can find in
polynomial time a
$(1-(1-1/s)^s)(1 - 5\epsilon)$ approximate solution with confidence $1-5\delta$.
\end{theorem}
\begin{proof} 
The averaging oracle is monotone and submodular~\cite{KKT:KDD2003}
and therefore
greedy can efficiently recover a set $T$ such that
$\hat{F_x}(T) \geq (1-(1-1/s)^s) \max_{S \mid |S|\leq S}
\hat{F_x}(S)$.

By Lemma \ref{termination:lemma}, the wrapper terminates using at most $2\max\{r,r_x(\epsilon,\delta)\}+r_v$ with probably at least $1 - 2\delta$. Applying Theorem \ref{weak alg:theorem} with $\rho = (1-(1-1/s)^s)$, we get that $\I^\tau(T) \geq (1-(1-1/s)^s)(1-5\epsilon)\OPT^\tau_s$ with probability at least $1 - 3\delta$. Hence, 
with probability at least $1 - 5\delta$ the wrapper applied with greedy finds $(1-(1-1/s)^s)(1-5\epsilon)$-approximate solution
using $2\max\{r,r_x(\epsilon,\delta)\}+r_v$ simulations.
\end{proof}

\section{Variance bounds for dependent models} \label{bdependence:sec}
In this section we provide a proof for Corollary \ref{extendIC:coro}.
We consider a 
natural extensions of IC models, $b$-dependence, that allow for
some dependencies between edges and mixtures of IC and IGT models.   For these extensions, we establish upper bounds
of the form \eqref{factorc:eq}
on the variance of the reachability of a set of nodes.

We bound the variance by constructing for each dependent model a
corresponding IC model and then apply the variance upper bound
established in Section~\ref{varUBproof:sec} for IC models.

For mixture models we provide a generic derivation that bounds the
variance of the mixture by variance of components.

\subsection{$b$-dependence models}
The first family we consider are
{\em $b$-dependence} models, which we define
as follows. We assume that all edges 
with the same tail node are partitioned into disjoint groups
where each  group is of size 
 at most $b$.
  The edges of each group $B$ are either all active together with probability $p_B$ or none is active with probability $1-p_B$.
The special case where all groups are of size $1$
 corresponds to an IC model (where all edges are independent). 
 

\begin{theorem}
Let $\mathcal{G}$ be a $b$-dependence model for some $b \geq 1$.
For every set $T$ we have that:
\[ \Var[\RReach^{\tau}_{\mathcal{G}}(T)] \leq 2b \tau \I^{\tau}_{\mathcal{G}}(T) \max_{v \in V \setminus T} {\I^{\tau}_{\mathcal{G}}(v)} \]
\end{theorem}\label{variance in dependence mode}
\begin{proof}
We construct an IC model $\mathcal{G}'$ from the given $b$-dependence model
$\mathcal{G}$.
The model $\mathcal{G}'$
is defined over the set of nodes $V$ of $\mathcal{G}$ together with an additional set $D$ of dummy nodes.
The construction has the properties that
$2\tau$-step influence in $\mathcal{G}'$ from a set of nodes $T\subseteq \ V$ is equal to $\tau$-stepped influence of $T$ in $\mathcal{G}$.
Furthermore, the variances
 of  the sizes of the 
 $2\tau$-step reachbility of $T$ in $\mathcal{G}'$
is the same as the variance of the $\tau$ step reachability of $T$ in $\mathcal{G}$.
The influence of each 
dummy node in $\mathcal{G'}$ is at most $b \max_{v \in V} \I^\tau(v)$.  The claim follows from these properties and Theorem \ref{var_upper_bound:thm}.

Here is a formal description of our reduction.

We start by putting in 
 $\mathcal{G}'$ 
 the set $V$ of the nodes of $\mathcal{G}$.  
Then 
for every group  $B = \{(u, v_1), (u, v_2), ..., (u, v_\ell)\}$  in $\mathcal{G}$ we do the following:
\begin{enumerate}
    \item Add a new dummy node $v_{B}$ to $\mathcal{G}'$, and add to $\mathcal{G}'$ the edge $(u, v_B)$ 
    and give it the probability $p_B$.
We assign weight $0$ to $v_B$ so that it does not contribute to the reachability of any set of nodes.

    \item we create edges $(v_B, v_i)$ for every $1\le i \le \ell$, each such edge has probability 1. 
\end{enumerate}

Let $T\subset V$ be a set of  nodes in $\mathcal{G}$.
It follows from our construction that 
for any set of nodes $B\subset V$
the probability that 
$\RReach^{2\tau}_{\mathcal{G}'}(T) = B$
is the same as the probability that
$\RReach^{\tau}_{\mathcal{G}}(T) = B$.
This implies that 
for any $T\subseteq V$
\[ \I^{2\tau}_{\mathcal{G}'}(T) =
\I^{\tau}_{\mathcal{G}}(T)\ , 
\]
and 
\[\Var[\RReach^{2\tau}_{\mathcal{G}'}(T)] = 
\Var[\RReach^{\tau}_{\mathcal{G}}(T)] \ .
\]

Each  dummy node is connected to at most $k$ original nodes, hence, 
 $ \I^{2\tau}_{\mathcal{G}'}(v)$ is bounded by 
$b \max_{v \in V} \I^{2\tau}_{\mathcal{G}'}(v)$. By Theorem~\ref{bound variance in graph} it follows that for every set of nodes $T$ in $\mathcal{G'}$:
\[ \Var[\RReach^{2\tau}_{\mathcal{G}'}(T)] \leq 2\tau \I^{2\tau}_{\mathcal{G}'}(T)  \max_{v \in V \setminus T} {\I^{2\tau}_{\mathcal{G}'}(v) . } \]

Combining all these observations together, we get that
\[ \Var[\RReach^{\tau}_{\mathcal{G}}(T)] \leq 2b \tau \I^{\tau}_{\mathcal{G}}(T) \max_{v \in V \setminus T} {\I^{\tau}_{\mathcal{G}}}(v) \]
\end{proof}

This Theorem can be generalized to  more complex  dependencies.
For example it holds for 
any distribution on subsets of the outgoing edges from each node that we can  realize by a distribution on disjoint subsets where we draw each subset  with certain probability, and take the union of the subset which we draw.

\subsection{Mixture of IC and IGT models}
The second family of dependent models we consider is a mixture of IC
and IGT models.

Consider a set of models $\mathcal{G}_i(V)$ for $i\in [r]$ and respective probabilities $p_i$ such that $\sum_{i=1}^r p_i =1$. 
We define a mixture model $\mathcal{G}(V)$ as follows. 
To draw $\bphi \sim \mathcal{G}$,  we first draw $i\in [r]$ according to
probabilities $p_i$ and then return $\bphi\sim \mathcal{G}_i$.

We provide two proofs for the variance bound of the mixture.  
The first is direct  and applies to any mixture of models that satisfies the variance
bound of Theorem~\ref{var_upper_bound:thm}), and in particular to
mixtures of strongly submodular SDMs.  The second proof is specific
to live-edge models and based on a reduction to an IC model.
\begin{theorem}
Consider a model $\mathcal{G}$ that is a mixture of $r$ models
$\mathcal{G}_i$ with probabilities $p_i$ that satisfy the variance
bound of Theorem~\ref{var_upper_bound:thm}.  Then for all 
$T\subset V$, 
\[ \Var[\RReach^{\tau}_{\mathcal{G}}(T)] \leq 
\frac{\tau + 1}{\min_i p_i} \I^\tau_{\mathcal{G}}(T) \max_{v \in V} {\I^{\tau}_{\mathcal{G}}(v)}\]

\end{theorem}
\begin{proof}

We first relate the influence of $T$ in the mixture model to the influence of $T$ in the components.
\begin{equation} \label{mixsum:eq}
\I^\tau_{\mathcal{G}}(T) = \E[\RReach^{\tau}_{\mathcal{G}}(T)]=  \sum_{i=1}^r p_i \E[\RReach^{\tau}_{\mathcal{G}_i}(T)]=  \sum_{i=1}^r p_i \I^{\tau}_{\mathcal{G}_i}(T)\ .
\end{equation}
This holds to any set $T$ and any $\tau$.  Therefore we also obtain the inequality
\begin{equation} \label{Mbound:eq}
\M^\tau_{\mathcal{G}}(\bar{T}) = \max_{v\in V\setminus T} I^\tau_{\mathcal{G}}(v) = \max_{v\in V\setminus T} \sum_{i=1}^r p_i \I^{\tau}_{\mathcal{G}_i}(v) \leq
\sum_{i=1}^r p_i \max_{v\in V\setminus T} \I^{\tau}_{\mathcal{G}_i}(v)=
\sum_{i=1}^r p_i \M^{\tau}_{\mathcal{G}_i}(\bar{T})\ .
\end{equation}
It also follows that  we can bound the influence values on the component by the respective ones in the mixture:
$\I^{\tau}_{\mathcal{G}_i}(T) \leq \frac{1}{p_i} \I^\tau_{\mathcal{G}}(T)$ and thus
\begin{eqnarray} 
    \max_i \I^{\tau}_{\mathcal{G}_i}(T) &\leq& \frac{1}{\min_i p_i} \I^\tau_{\mathcal{G}}(T) \label{componentIbound:eq} \\
    \max_i \M^{\tau}_{\mathcal{G}_i}(\overline{T}) &\leq& \frac{1}{\min_i p_i} \M^\tau_{\mathcal{G}}(\overline{T})\ . \label{componentMbound:eq}
\end{eqnarray}

The random variable $\RReach^{\tau}_{\mathcal{G}}(T)$  can be expressed as a sum of 
of $r$ products of random variables:
$$\RReach^{\tau}_{\mathcal{G}}(T) = \sum_{i=1}^r X_i \RReach^{\tau}_{\mathcal{G}_i}(T)\ ,$$ where $X_i$ are Bernoulli with probabilities $p_i$.  The random variables  $\{\RReach^{\tau}_{\mathcal{G}_i}(T)\}$
are independent of each other and also are independent from (the joint distribution of) $\{X_i\}$.  The variables $\{X_i\}$ have negative dependence as $\sum_i X_i =1$ and thus
the products 
$X_i \RReach^{\tau}_{\mathcal{G}_i}(T)$
are also negatively dependent and thus
\[\Var[\RReach^{\tau}_{\mathcal{G}}(T)] \leq \sum_i \Var[X_i \RReach^{\tau}_{\mathcal{G}_i}(T)]\ .\]

We will instead bound the variance of a surrogate random variable 
\[ Y = \sum_i X_i \RReach^{\tau}_{\mathcal{G}_i}(T) \] that has the same sum of products but with the variables $\{X_i\}$ being independent of each other and hence the products are also independent. 
We have 
\begin{equation} \label{varinequ:eq}
\Var[Y] = \sum_i \Var[X_i \RReach^{\tau}_{\mathcal{G}_i}(T)] \geq \Var[\RReach^{\tau}_{\mathcal{G}}(T)] 
\end{equation}

We next express the variance of each product using variance properties of the product of two independent random variables. For $i\in [r]$:
\begin{eqnarray*}
\Var[X_i \RReach^{\tau}_{\mathcal{G}_i}(T)] &=& \Var[X_i]\Var[\RReach^{\tau}_{\mathcal{G}_i}(T)]+
\E[X_i]^2\Var[\RReach^{\tau}_{\mathcal{G}_i}(T)]^2+
\Var[X_i]\E[\RReach^{\tau}_{\mathcal{G}_i}(T)]^2 \\
&=& p_i(1-p_i) \Var[\RReach^{\tau}_{\mathcal{G}_i}(T)] + p_i^2  \Var[\RReach^{\tau}_{\mathcal{G}_i}(T)] + p_i(1-p_i)\I^\tau_{\mathcal{G}_i}(T)^2 \\
&=& p_i \Var[\RReach^{\tau}_{\mathcal{G}_i}(T)] + p_i(1-p_i) 
\I^\tau_{\mathcal{G}_i}(T)^2\ .
\end{eqnarray*}

Therefore,
invoking Theorem~\ref{var_upper_bound:thm} to bound the variance 
for each IC model $\mathcal{G}_i$ and then using \eqref{componentIbound:eq} and  \eqref{componentMbound:eq} and finally using \eqref{mixsum:eq} and \eqref{Mbound:eq} we get
\begin{eqnarray*}
\Var[Y]&=& \sum_{i=1}^r p_i \left( \Var[\RReach^{\tau}_{\mathcal{G}_i}(T)] + p_i(1-p_i) 
\I^\tau_{\mathcal{G}_i}(T)^2\right)\\
&\leq& \sum_{i=1}^r p_i \Var[\RReach^{\tau}_{\mathcal{G}_i}(T)] +  \sum_{i=1}^r p_i \I^\tau_{\mathcal{G}_i}(T)^2\\
&\underbrace{\leq}_\text{Theorem~\ref{var_upper_bound:thm}}& \sum_{i=1}^r p_i \tau \M^{\tau-1}_{\mathcal{G}_i}(\bar{T}) \I^\tau_{\mathcal{G}_i}(T)+  \sum_{i=1}^r p_i \I^\tau_{\mathcal{G}_i}(T)^2\\
&\underbrace{\leq}_\text{\eqref{componentIbound:eq},  \eqref{componentMbound:eq}}& \frac{\tau}{\min_i p_i}\I^\tau_{\mathcal{G}}(T) \sum_{i=1}^r p_i  \M^{\tau-1}_{\mathcal{G}_i}(\bar{T}) +  \frac{1}{\min_i p_i}\I^\tau_{\mathcal{G}}(T)\sum_{i=1}^r p_i \I^\tau_{\mathcal{G}_i}(T) \\
&\underbrace{\leq}_\text{\eqref{mixsum:eq} , \eqref{Mbound:eq}}& \frac{1}{\min_i p_i}\I^\tau_{\mathcal{G}}(T) \left( \tau \M^{\tau-1}_{\mathcal{G}}(\bar{T}) + \I^\tau_{\mathcal{G}}(T)\right) \leq \frac{\tau + 1}{\min_i p_i} \I^\tau_{\mathcal{G}}(T) \max_{v \in V} {\I^{\tau}_{\mathcal{G}}(v)}
\end{eqnarray*}

\end{proof}

We next give a different proof (of a slightly different bound) for
live-edge models using a reduction to an IC model. 
Consider a set of 
$\tau$-steps models $\mathcal{G}_i(V,\mathcal{E}_i)$ for $i\in [r]$ and respective probabilities $p_i$ such that $\sum_{i=1}^r p_i =1$. 
We define a mixture model $\mathcal{G}(V,\bigcup_i \mathcal{E}_i)$ as follows. 
To draw $E\sim \mathcal{G}$,  we first draw $i\in [r]$ according to probabilities $p_i$ and then return $E\sim \mathcal{G}_i$. 
\begin{theorem}
Consider a model $\mathcal{G}$ that is a mixture of $r$ IC models $\mathcal{G}_i$ with probabilities $p_i$. Then for all 
$T\subset V$, 
\[ \Var[\RReach^{\tau}_{\mathcal{G}}(T)] \leq  \frac{\tau+1}{\min_i p_i} \I^{\tau}_{\mathcal{G}}(T) \max\{ \I^\tau_{\mathcal{G}}(T),\max_{v \in V} {\I^{\tau}_{\mathcal{G}}(v)} \}\]

\end{theorem}
\begin{proof}
We first argue that we can assume without loss of generality that $T$ is a single node and $\bigcup_i \mathcal{E}_i$ does not contain edges that are incoming to $T$.
We can transform a general case $\mathcal{G}$ and $T$  to this form by contracting all nodes in $T$ into a single node and deleting all edges that are incoming to $T$. We then retain the same conditional distribution on the remaining edges.  Note that this transformation preserves the distribution of 
$\RReach^{\tau}_{\mathcal{G}}(T)$ and hence also its expectation and variance.
The influence values 
$I^{\tau}_{\mathcal{G}}(v)$ of nodes $v\in V\setminus T$ can only decrease.  Finally, the transformed model is also a mixture of correspondingly transformed IC models, where in each such model 
the distribution of 
$\RReach^{\tau}_{\mathcal{G}_i}(T)$ remains the same and influence values
 $I^{\tau}_{\mathcal{G}_i}(v)$ can only decrease.  It follows that the claimed variance bound for the transformed model implies the same bound for the original model.

 We construct a new IC model $\mathcal{G'}$ with respect to (a single node) $T$ as follows.
 The new model has nodes
 $V' = \{v\} \cup \bigcup_i V_i$, where each $V_i$ is a map of $V$.
 We create an instantiation of each of our IC models $\mathcal{G}_i$ with set of nodes $V_i$ and edges $\mathcal{E}_i$ with the probabilities as in the model $\mathcal{G}_i$. 
 The new IC model $\mathcal{G'}$ has a root node $v$ with weight $0$ and for each $i\in [r]$, there is an edge $(v,T_i)$ with probability $p_i$, where $T_i$ is the image of $T$ in the copy of $\mathcal{G}_i$.  We can see that 
 \begin{equation} \label{mixsum:eq}
 I^{\tau+1}_{\mathcal{G'}}(v) = \I^\tau_{\mathcal{G}}(T) = \sum_{i=1}^r p_i \I^{\tau}_{\mathcal{G}_i}(T)\ ,
\end{equation}
that is, the $\tau+1$ steps influence of $v$ in the constructed IC model $\mathcal{G'}$ is equal to the $\tau$ steps influence of $T$ in the mixture model $\mathcal{G}$. 

We next consider the variance of the random variables  $\RReach^{\tau}_{\mathcal{G}}(T)$ and 
$\RReach^{\tau+1}_{\mathcal{G'}}(v)$.  Both these random variables are a sum of $r$ products of random variables:
$$\sum_{i=1}^r X_i \RReach^{\tau}_{\mathcal{G}_i}(T)\ ,$$ where $X_i$ are Bernoulli with probabilities $p_i$.  In both cases the random variables  $\{\RReach^{\tau}_{\mathcal{G}_i}(T)\}$
are independent of each other and also are independent from (the joint distribution of) $\{X_i\}$.  
But in the case of $\RReach^{\tau+1}_{\mathcal{G'}}(v)$ the random variables $X_i$ are independent and hence also the products are independent and
in the case of
$\RReach^{\tau}_{\mathcal{G}}(T)$,  the variables $\{X_i\}$ have negative dependence as $\sum_i X_i =1$ and thus
the products 
$X_i \RReach^{\tau}_{\mathcal{G}_i}(T)$
are also negatively dependent. Therefore, 
\begin{equation} \label{varinequ:eq}
\Var[\RReach^{\tau}_{\mathcal{G}}(T)] \leq \Var[\RReach^{\tau+1}_{\mathcal{G'}}(v)] = \sum_i \Var[X_i \RReach^{\tau}_{\mathcal{G}_i}(T)]\ .
\end{equation}
Finally, we bound
$\M^{\tau}_{\mathcal{G}'}(\bar{v})$ by considering the maximum influence of a node other than $v$ in the constructed model $\mathcal{G}'$.
For $T_i$ we have  
\begin{equation} \label{subrootbound:eq} \I^{\tau}_{\mathcal{G}'}(T_i)=
\I^{\tau}_{\mathcal{G}_i}(T)
\leq \frac{1}{p_i} 
\I^\tau_{\mathcal{G}}(T) \ ,
\end{equation}
where the last inequality follows from \eqref{mixsum:eq}.
We next consider nodes $z_i \in V_i$ that is a map of a node $z\in V$.  
\begin{equation} \label{mapnodebound:eq}
\I^{\tau}_{\mathcal{G}'}(z_i)=
\I^\tau_{\mathcal{G}_i}(z)
\leq \frac{1}{p_i} \I^{\tau}_{\mathcal{G}}(z)\ .
\end{equation}
The last inequality follows because for
any node $z\in v$ we have
$\I^\tau_{\mathcal{G}}(z) = \sum_{i=1}^r p_i I^\tau_{\mathcal{G}_i}(z)$. 
Combining \eqref{subrootbound:eq} and \eqref{mapnodebound:eq} we get
\begin{equation} \label{Mbound:eq}
 \M^{\tau}_{\mathcal{G}'}(\bar{v}) = \max_{u\in V'\setminus \{v\}} \I^{\tau}_{\mathcal{G}'}(u) \leq \max_{u\in V}\max_i  \I^\tau_{\mathcal{G}_i}(u) \leq  \max_{u\in V} \frac{1}{\min_i p_i} \I^\tau_{\mathcal{G}}(u) = 
\frac{1}{\min_i p_i} \max\{I^\tau_{\mathcal{G}}(T),\max_{z\in V} \I^{\tau}_{\mathcal{G}}(z)\}\ .
\end{equation}
To conclude, we invoke Theorem~\ref{var_upper_bound:thm} for the IC model $\mathcal{G}'$:
\[\Var[\RReach^{\tau}_{\mathcal{G}}(T)]   \underbrace{\leq}_\text{\eqref{varinequ:eq}} \Var[\RReach^{\tau+1}_{\mathcal{G'}}(v)] \underbrace{\leq}_\text{Theorem~\ref{var_upper_bound:thm}} (\tau+1) I^{\tau+1}_{\mathcal{G}'}(v) \M^{\tau}_{\mathcal{G}'}(\bar{v})\ .
\]
We then apply inequalities \eqref{Mbound:eq} and the equality \eqref{mixsum:eq} to obtain the claim.
 \end{proof}











 \end{document}